\documentclass[runningheads,envcountsame]{llncs}
\usepackage[T1]{fontenc}
\usepackage{graphicx}
\usepackage[utf8]{inputenc}

\usepackage{comment}
\usepackage{xcolor}

\usepackage{amsmath,amssymb,amsfonts,mathrsfs}

\usepackage{array}

\usepackage{mathtools}
\usepackage{enumitem}

\usepackage{algpseudocode}
\usepackage{algorithm}

\usepackage{hyperref}
\usepackage{todonotes}
\usepackage{xspace}
\usepackage{cite}

\usepackage{subfig}

\newcommand{\ooea}{$(1 + 1)$-EA\xspace}

\newcommand{\cga}{\textsc{cGA}\xspace}
\newcommand{\cGA}{\textsc{cGA}\xspace}

\newcommand{\onemax}{\textsc{OneMax}\xspace}

\newcommand{\dynBV}{\textsc{Dynamic BinVal}\xspace}
\newcommand{\DBV}{\textsc{DynBV}\xspace}

\newcommand{\DeceptiveLeadingBlocks}{\textsc{DeceptiveLeadingBlocks}\xspace}

\newcommand{\N}{\mathbb{N}}

\newcommand{\R}{\mathbb{R}}

\renewcommand{\epsilon}{\ensuremath\varepsilon}
\newcommand{\eps}{\ensuremath{\varepsilon}}
\DeclareMathOperator{\polylog}{\mathrm{polylog}}

\renewcommand{\phi}{\ensuremath{\varphi}}

\newcommand{\prob}{\mathrm{Pr}}
\newcommand{\expe}{\mathbb{E}}

\def\setof#1{\left\{#1  \right\}}

\begin{document}
\title{Faster Optimization Through Genetic Drift}
\titlerunning{\dynBV with the \cGA}
%
\author{Cella Florescu \and
Marc Kaufmann \and
Johannes Lengler
\and Ulysse Schaller}
\authorrunning{C. Florescu, M. Kaufmann, J. Lengler, and U. Schaller}
%
\institute{Department of Computer Science, ETH Zürich, Zürich, Switzerland. 
\email{\{cella.florescu, marc.kaufmann, johannes.lengler, ulysse.schaller\}@inf.ethz.ch}\\}
\maketitle              
\begin{abstract}

The compact Genetic Algorithm (\cGA), parameterized by its hypothetical population size $K$, offers a low-memory alternative to evolving a large offspring population of solutions. It evolves a probability distribution, biasing it towards promising samples. For the classical benchmark \onemax, the \cGA has to two different modes of operation: a conservative one with small step sizes $\Theta(1/(\sqrt{n}\log n))$, which is slow but prevents genetic drift, and an aggressive one with large step sizes $\Theta(1/\log n)$, in which genetic drift leads to wrong decisions, but those are corrected efficiently. On \onemax, an easy hill-climbing problem, both modes lead to optimization times of $\Theta(n\log n)$ and are thus equally efficient.

In this paper we study how both regimes change when we replace \onemax by the harder hill-climbing problem \dynBV. It turns out that the aggressive mode is not affected and still yields quasi-linear runtime $O(n\polylog n)$. However, the conservative mode becomes substantially slower, yielding a runtime of $\Omega(n^2)$, since genetic drift can only be avoided with smaller step sizes of $O(1/n)$. We complement our theoretical results with simulations.

\keywords{compact Genetic Algorithm \and Genetic Drift \and Estimation-of-Distribution Algorithm \and Dynamic Binary Value}
\end{abstract}

\section{Introduction}\label{sec:intro}

Estimation-of-distribution algorithms (EDAs) are a family of randomized optimization heuristics in which the algorithm evolves a probability distribution over the search space. In each iteration, it samples solutions from this distribution, evaluates their quality (also called fitness), and updates the probability distribution accordingly. Examples in discrete domains include the \cga, \textsc{UMDA}, \textsc{PBIL}, ant colony systems like the \textsc{MMAS}, and multivariate systems like \textsc{hBOA}~\cite{pelikan2004parameter} or \textsc{MIMIC}~\cite{de1996mimic}, see~\cite{krejca2020theory} for a survey. EDAs turn out to be a powerful alternative to population-based heuristics like evolutionary and genetic algorithms. They have the advantage that they often sample from a wider region of the search space than their population-based alternatives, which makes them less susceptible to local deviations like (frozen or non-frozen) noise and local optima~\cite{friedrich2016compact,lehre2019runtime,witt2023majority,doerr2021runtime,friedrich2022theoretical}. 

EDAs have been used for several decades, but theoretical investigations of EDAs have only started to gain momentum a few years ago. While for some aspects a clear picture has emerged, like that EDAs are able to cope with large amounts of noise~\cite{friedrich2016compact}, there is one aspect for which researchers have found a complex and ambiguous pattern: genetic drift. Genetic drift is the tendency of an algorithm to move through the fitness landscape even in absence of a clear signal-to-noise ratio.\footnote{The term \emph{drift} is also used in the context of \emph{drift analysis}, where it means the expected change, which is almost the opposite concept. The term ``genetic drift'' should not be confused with this other meaning of the term ``drift''.} While it is possible to avoid genetic drift by tracking the signal-to-noise ratio~\cite{doerr2018significance,doerr2020sharp}, this conservative attempt of avoiding mistakes could potentially make the algorithm slow and inflexible. An alternative approach might be to embrace genetic drift, allow the algorithm to swiftly move through the search space, and let it correct mistakes as they appear. 

Indeed, these two alternatives are exemplified by the behaviour of the compact Genetic Algorithm \cGA on the pseudo-Boolean function \onemax~\cite{lengler2021complex}. The \onemax function assigns to a bit string $x\in \{0,1\}^n$ the number of one-bits in $x$. It is one of the simplest and most classical hill-climbing benchmarks. The \cGA maintains for each of the $n$ coordinates a \emph{frequency} $p_i$, which encodes the probability that the $i$-th bit is set to one in the distribution. In each iteration, it samples two solutions, and for each component $i$ it shifts the frequency $p_i$ by $1/K$ towards the value of the fitter of the two solutions, see Section~\ref{sec:setting} for full details. The step size $1/K$ determines how aggressively or conservatively the algorithm updates. It is well understood that the size of $K$ determines whether genetic drift happens to a relevant extent on \onemax or not. If $K = \omega(\sqrt{n}\log n)$ then the frequencies move so slowly that the signal exceeds the noise, and all frequencies move slowly but steadily towards the upper boundary. This corresponds to the regime where genetic drift is avoided, and we refer to this as \emph{conservative} regime. On the other hand, if $K = o(\sqrt{n}\log n)$ then the signal is weaker than the noise, and some bits move to the wrong boundary due to genetic drift.  In the subsequent optimization process, these mistakes are then slowly corrected. We call this the \emph{aggressive} regime.

It turns out that both regimes are equally efficient on \onemax. For suitable $K= C\log n$ with a large constant $C$, errors are corrected so quickly that the optimum is sampled in $O(n\log n)$ iterations.\footnote{This was only shown formally for the UMDA in~\cite{witt2023majority} and~\cite{dang2019level}, not for the \cga. However,~\cite{lengler2021complex} contains an informal argument why the results should also apply to the \cga.} On the other hand, if $K= C\sqrt{n}\log n$ with a large constant $C$, then the algorithm moves more slowly, but does not make any errors, which yields the same asymptotic runtime $O(n\log n)$. Both parameter settings are brittle with respect to smaller $K$: if either $K=C\log n$ or $K= C\sqrt{n}\log n$ are decreased only slightly, this results in a sudden loss of performance. On the other hand, if the parameter $K$ is increased from either $K=C\log n$ or $K= C\sqrt{n}\log n$, then the performance deteriorates slowly but steadily. Hence, there are two optimal parameter settings for the \cGA on \onemax, a conservative one which avoids genetic drift and an aggressive one which embraces genetic drift. 

Since the mentioned analysis was limited to \onemax, it remained open whether both modes of the algorithm also show comparable performance for other hill-climbing tasks. In this paper, we give a negative answer and show that for the harder hill-climbing problem \dynBV, the aggressive mode still finds the optimum in quasi-linear time, while the conservative mode needs time $\Omega(n^2)$.

\subsection{Our results}

We investigate the \cGA on the function \dynBV, or \DBV for short.  This function, introduced in~\cite{lengler2022large}, builds on the classical linear test function \textsc{Binary Value} that assigns to each binary string the integer that is represented by it in the binary number system. \DBV is obtained by drawing at each iteration a random permutation of the weights and then evaluating all solutions with the permuted \textsc{Binary Value} function, see Section~\ref{sec:setting} for a formal definition. \textsc{Binary Value} is conjectured to be the hardest linear function\footnote{For example for the \ooea. It is a famous open problem to prove this formally.}, and \DBV is known to be the hardest dynamic linear function\footnote{It is not formally a dynamic linear function in the sense of~\cite{lengler2018noisy}, but can be obtained as a limit of such functions~\cite{lengler2022large}.}~\cite{lengler2022large,lengler2021runtime}. This makes it the perfect benchmark for a hard hill-climbing task. For more discussion of the benchmark, see Section~\ref{sec:setup}.

\subsubsection{The conservative regime is slow.} Our first main result is the following lower runtime bound, which holds for all $K =O(\mathrm{poly}(n))$. In this range the runtime will at first increase quadratically in $K$, until $K$ reaches the dimension of the search space - at which point the runtime dependency becomes $\Omega(K \cdot n)$.

\begin{theorem}
\label{thm:lowerbound}
Let $\Bar{p}\in (0, \frac{1}{2})$ be arbitrary and consider the \cGA with parameter $K =O(\mathrm{poly}(n))$ and boundaries at $\Bar{p}$ and $1-\Bar{p}$ on \DBV. Then with high probability, the optimum is not sampled during the first $\Omega(K \cdot \min\{K, n\})$ iterations.
\end{theorem}
The reason for this is captured in Lemma~\ref{lemma:slow_progress}, which states that in this period, there are always a linear number of bits which stay in some constant interval around their initialization value, so that it is exponentially unlikely to sample the optimum.

If we want to avoid genetic drift, we have to choose a rather large $K$ (small step sizes) to overcome the small signal-to-noise ratio that is inherent in \DBV. The following theorem states that any $K = O(n)$ will lead to substantial genetic drift and hence belongs to the aggressive regime. 
This agrees with the guidelines from~\cite{doerr2021runtime} on how to avoid genetic drift.

\begin{theorem}
\label{thm:geneticdrift}
    For every $\rho>0$ and $\beta\in (\Bar{p},\frac{1}{2})$ there is $\delta>0$ such that the following holds. Consider the \cGA with parameter $K \le \rho n$ on \DBV. Then with high probability as $K \to \infty$ at least $\delta n$ frequencies drop below $\beta$ during optimization.
\end{theorem}

Theorem~\ref{thm:geneticdrift} shows that indeed the only possibility to avoid substantial genetic drift is to set $K= \omega(n)$, which leads to a runtime of $\omega(n^2)$ by Theorem~\ref{thm:lowerbound}. Hence, \DBV can not be optimized in quadratic time by any parameter setting of the \cGA that avoids genetic drift. As we will see below, aggressive parameter settings that allow genetic drift are much more efficient. Before we come to this other regime, we complement Theorem~\ref{thm:lowerbound} with a matching upper bound that holds when we are safely in the conservative regime with $K=\Omega(n \log n)$.  
To simplify the proof, we require a slight adjustment of the boundary values.\footnote{More precisely, we set them to $\frac{1}{cn}, 1-\frac{1}{cn}$ for a large enough constant $c>0$. As our simulations, which are all conducted with boundaries $\frac{1}{n}$, $1-\frac{1}{n}$, show, this choice does not affect the asymptotic behaviour.}
\begin{theorem}
    \label{thm:intro_upper_bound_runtime}
    Consider the \cGA with parameter $K=\mathrm{poly}(n)$ and boundaries at $\frac{1}{cn}$ and $1-\frac{1}{cn}$ on \DBV. If $K \geq c' \cdot n \log{n}$, and the constants $c,c'>0$ are large enough, then the expected optimization time is $O(K n)$. 
\end{theorem}

The main ingredient for this proof is to show that in the first polynomially many rounds, all frequencies will stay bounded away from the lower boundary (Proposition~\ref{prop:bound_stay_between_constants}). Hence, the proof of Theorem~\ref{thm:intro_upper_bound_runtime} is similar to other proofs of upper runtime bounds in conservative regimes~\cite{doerr2020sharp,witt2019upper}.

Together,  Theorems~\ref{thm:lowerbound} and~\ref{thm:intro_upper_bound_runtime} give tight runtime bounds of $\Theta(Kn)$ in the conservative regime. This implies in particular that the runtime in this regime is much larger for \DBV than for \onemax, where the runtime is $O(K\sqrt{n})$~\cite{sudholt2019choice}.

We remark that the different transition points between conservative and aggressive regime ($K=\omega(n)$ in Theorem~\ref{thm:intro_upper_bound_runtime} and $K= \Omega(n \log n)$ in Theorem~\ref{thm:intro_upper_bound_runtime}) are natural because there are different possible definitions of the conservative regime: that no frequency drops below $1/3$ (or any other fixed constant below $1/2$), that no frequency reaches the lower boundary, or that the number of frequencies hitting the lower boundary is sublinear. All these variants lead to different transition points between the conservative and aggressive regime.

\subsubsection{The aggressive regime is fast.} Our second result shows that in contrast, the optimization time of the \cGA remains quasi-linear for small $K$, i.e. linear up to a poly-logarithmic factor. This corresponds to the aggressive regime where many frequencies reach the wrong boundary, but those errors are corrected efficiently. To make the analysis simpler, similar as for Theorem~\ref{thm:intro_upper_bound_runtime}, we do not set the two boundary values at their standard values $1/n$ and $1-1/n$, but this time we even set them to $1/(n\polylog n)$ and $1-1/(n\polylog n)$. Moreover, we do not use the smallest possible (most aggressive) parameter choice $K=C\log n$ for the aggressive regime, but rather choose the slightly more conservative $K=\Theta(\log^2 n)$. Then we prove the following result. 

\begin{theorem}
\label{thm:boundslogn}
    Consider the \cGA with parameter $K = \Theta(\log^2 n)$ and boundaries $1/(n\log^7 n)$ and $1-1/(n\log^7 n)$ on \DBV. Then the optimum is sampled in $O(n \cdot \mathrm{polylog}(n))$\footnote{Our derived bound yields $O(n \log^{16} n)$ but this is not tight with regard to $\log$-factors in various places, so we did not optimize for the exponent of the logarithm. } iterations with probability $1-o(1)$.
\end{theorem}

We note that we made no effort to optimize the exponent $7$ of the poly-logarithmic factor in the boundaries. We conjecture that the true runtime for optimal parameters is $O(n\log n)$, and that this is achieved with the standard boundaries $1/n$ and $1-1/n$. Notably, this would mean that there is no substantial runtime difference in the aggressive regime with optimal parameters between \onemax and \DBV, in stark contrast to the conservative regime. We do not quite show this statement, but we show it up to poly-logarithmic factors.

Compared to this conjecture, our analysis is likely not tight in several ways. Firstly, even for the given parameters we believe that the poly-logarithmic exponent of our runtime bound could be reduced at the cost of a more technical analysis. Secondly, both the conservative choice of $K$ and the non-standard choice of the threshold likely bring us away from the optimal parameter. This simplifies the proof, but costs us performance, even though only logarithmic factors. We suspect that the optimal parameter setup is indeed the standard setup of $K=C\log n$ for a large constant $C$ and boundaries $1/n$ and $1-1/n$. This is supported by the experiments presented in Section~\ref{sec:sim}.

\subsection{Discussion of the setup and related work}\label{sec:setup}

\paragraph{Signal steps and \DBV.} 
In order to understand genetic drift, a key question is how often each frequency receives a signal step. We call an iteration a \emph{signal step} for frequency $p_i$ if both solutions differ in this position $i$ and the two values of this position are necessary to decide which of the two solution is fitter. When both solutions differ, but their values are irrelevant for identifying the fitter solution, then we call the iteration a \emph{random walk step} for frequency $p_i$. In the initial phase of the \cGA on \onemax, the probability of a signal step is $\Theta(1/\sqrt{n})$. For \DBV, the signal probability is considerably weaker, namely of order $\Theta(1/n)$. In fact, in each iteration exactly one frequency receives a signal step, except when the two offspring sampled by the algorithm agree in every single bit. 

As mentioned, \DBV is the hardest dynamic linear function. This also holds in terms of the signal strength: when comparing two non-equal solutions, then exactly one frequency gets a signal, while all other frequencies perform a random walk step. This is the weakest signal strength among all dynamic linear functions, and even the hardest among all dynamic monotone functions~\cite{kaufmann2023hardest}, since every monotone function, static or dynamic, will always provide a signal step to at least one frequency when comparing two solutions. 

Although it is a dynamic function, \DBV provides a hill-climbing task in the sense that in each iteration and at any position, a one-bit gives a higher fitness than a zero-bit. Thus, pure hill-climbing heuristics such as Random Local Search (RLS) can be highly efficient on this function. Moreover, \DBV is more symmetric than the classical \textsc{Binary Value} function, which makes the analysis simpler. All these properties make \DBV the perfect benchmark for a theoretical runtime analysis of a hard hill-climbing task.

\paragraph{Related work.}

It has been shown that \dynBV is harder to optimize by evolutionary algorithms than static monotone functions in various ways. The $(1,\lambda)$-EA with self-adapting offspring population size fails on \dynBV while succeeding on \onemax if the hyperparameters are not set correctly~\cite{kaufmann2023onemax}. Furthermore, a ``switching'' variant of \dynBV minimizes drift in the number of zeros at every search point for the $(1+1)$-EA for any mutation rate at every search point, making it harder to optimize than any \textit{static} monotone function~\cite{kaufmann2023hardest}.

We do not claim that the aggressive mode of the \cGA is generally superior to the conservative mode. However, our results show that the other extreme position of avoiding genetic drift at all costs, does cost performance for \dynBV. On the other hand, the conservative mode was shown to be superior on the function \DeceptiveLeadingBlocks for some parameter settings~\cite{lehre2019limitations,doerr2021univariate}, though a discussion at a Dagstuhl seminar shows that opinions are split about the implications of these results~\cite{dagstuhl2022estimation}.  We hope that further research will give a clearer and more nuanced picture on the benefits and drawbacks of genetic drift.

\subsection{Overview of the paper}
In Section~\ref{sec:setting} we describe the \cGA and our benchmark \dynBV, followed by important terminology and technical tools which we will use throughout the paper. Section~\ref{sec:dyn} characterizes the dynamics of the marginal probabilities which are used in the generation of offspring. In Section~\ref{sec:lower_bound}, we prove the lower runtime bound, Theorem~\ref{thm:lowerbound}, and Theorem~\ref{thm:geneticdrift}, thus establishing that any $K=O(n)$ will lead to genetic drift. In the subsequent Section~\ref{sec:nlogn}, we prove Theorem~\ref{thm:intro_upper_bound_runtime}, the upper runtime bound for the conservative regime. Section~\ref{sec:logn} establishes an upper runtime bound for the aggressive regime, culminating in Theorem~\ref{thm:boundslogn}. The paper concludes with the simulations in Section~\ref{sec:sim}.

\section{Setting}
\label{sec:setting}
Our search space is always $\{0,1\}^n$. We say that an event $\mathcal{E}=\mathcal{E}(n)$ holds \textit{with high probability} or \textit{whp} if $\Pr[\mathcal{E}] \to 1$ as $n \to \infty$. We may for simplicity omit the parameter $t$ indicating the iteration when it is clear from context.

\subsection{The Algorithm: the \cGA with hypothetical population size $K$}
\label{sec:algo}

We begin with an intuitive description of the \cGA. Before the start of the algorithm, we fix a capping probability $\Bar{p}<\frac{1}{2}$ and a hypothetical population size $K$, which we may think of as an inverse update strength. At every iteration, the algorithm generates two offspring $x$ and $y$ independently of each other by the same sampling procedure. At iteration $t$, the $i$-th bit of the offspring to be sampled is set (independently of all other bits and of all previous iterations) to $1$ with probability $p_{i,t}$ and set to $0$ otherwise. The probabilities are initialized to $\frac{1}{2}$ for all bits and evolve according to the following procedure: At each iteration, the fitness of $x$ and $y$ are compared according to the fitness function $f$ - in our case this will be \dynBV. If the fitter offspring contains a $1$ at position $i$, we increase the probability of sampling a $1$ bit, $p_i$, by $\frac{1}{K}$ for the next iteration, otherwise we decrease it by the same amount. If there is no strictly fitter offspring, i.e. $f(x)=f(y)$, all probabilities $p_{i,t}$ remain unchanged in the next iteration. To ensure the algorithm does not get stuck by fixing one of the bits, i.e. sampling a $1$ with probability $1$  or sampling a $0$ with probability $1$, we restrict the possible values for probabilities $p_{i,t}$ to the interval $[\Bar{p}, 1-\Bar{p}]$. If an update step would make a $p_{i,t}$ exceed these bounds, we set it to the boundary value instead. The algorithm stops when the optimum has been sampled (as one of the two offspring in a given iteration). The full pseudocode is provided in Algorithm \ref{algo:cga}.

\begin{algorithm}
    \caption{$\cga(f, K, \Bar{p})$}
    \label{algo:cga}
    \begin{algorithmic}
        
        \State $t \gets 0$
        \State $p_{1, t} \gets p_{2, t} \gets \cdots \gets p_{n, t} \gets \frac{1}{2}$
        
        \While{optimum has not been sampled} 
            \For{$i \in \{1, 2, \dots, n\}$}
                \State $x_i \gets 1$ with probability $p_{i, t}$ and 0 otherwise
                \State $y_i \gets 1$ with probability $p_{i, t}$ and 0 otherwise
            \EndFor

            \If{$f(x) = f(y)$}
                \State $t \gets t + 1$
                \State \textbf{continue}
            \ElsIf{$f(x) < f(y)$}
                \State swap $x$ and $y$
            \EndIf

            \For{$i \in \{1, 2, \dots, n\}$}
                \If{$x_i > y_i$}
                    \State $p_{i, t + 1}' \gets p_{i, t} + \frac{1}{K}$
                \ElsIf{$x_i < y_i$}
                    \State $p_{i, t + 1}' \gets p_{i, t} - \frac{1}{K}$
                \Else
                    \State $p_{i, t + 1}' \gets p_{i, t}$
                \EndIf
                \State $p_{i, t + 1} \gets \min{\left\{\max{\left\{\Bar{p}, p_{i, t + 1}'\right\}}, 1 - \Bar{p}\right\}}$
            \EndFor
            
            \State $t \gets t + 1$
        \EndWhile
    \end{algorithmic}
\end{algorithm}

\subsection{The Benchmark: \dynBV}

The benchmark we consider is the \dynBV function.

In \dynBV, at each iteration $t$, we draw uniformly at random (and independently of everything else) from the set of bijections from $\{1, \dots, n\}$ onto itself an element $\pi_t \colon \{1, 2, \dots, n\} \to \{1, 2, \dots, n\}$. Note that this can be seen as a permutation of the bits of the search point. The fitness function for iteration $t$ is then given by
\[
f_t(x) = \sum_{i = 1}^n{2^{n - i} \cdot x_{\pi_t(i)}}.
\]
Intuitively, the offspring which has a $1$ bit at the most significant position (given by the permutation $\pi_t$) at which the two offspring differ is considered fitter.

\subsection{Terminology}\label{sec:signal}

Below, we introduce some useful concepts and terminology that will be used throughout the paper. 

\paragraph{Signal step.} A signal step of a bit $i \in \{1, 2, \dots, n\}$ is the increase in the marginal probability of bit $i$ during an iteration $t$ where the value of this bit in the two offspring was decisive. In other words, bit $i$ performs a signal step at iteration $t$ if and only if the two offspring differ at bit $i$ and are equal at any other bit $i'$ for which $\pi_t(i') < \pi_t(i)$.

\paragraph{Random step.} A random step of a bit $i \in \{1, 2, \dots, n\}$ is the change in the marginal probability of bit $i$ during an iteration where the value of this bit in the two offspring was not decisive. Thus, all changes in a bit's marginal probability that are not signal steps are random steps.

\paragraph{Sampling variance.} The sampling variance at time $t$ is the variance of the binomial distribution induced by the probabilities $p_{i,t}, i=1,\dots n$. We denote it by $V_t \coloneqq \sum_{i = 1}^n{p_{i, t} (1 - p_{i, t})}$. Intuitively, it is the sum at time $t$ of the variances contributed by the probability of each frequency in the generating distribution. 

\paragraph{Lower/Upper boundary.} Recall from subsection~\ref{sec:algo} that the possible values for the probabilities $p_{i,t}$ are restricted to an interval $[\Bar{p}, 1-\Bar{p}] \subsetneq [0,1]$ for some value $\Bar{p}=O(\frac{1}{n})$ to ensure that the algorithm does not get stuck. The values $\Bar{p}$ and $1 - \Bar{p}$ are going to be referred to as lower and upper boundary respectively. For example, we say a frequency $p_{i, t}$ is at its upper boundary if $p_{i, t} = 1 - \Bar{p}$. Note that the distances of the lower and upper boundaries from $0$ and $1$ is always the same, respectively. $\Bar{p}$ is fixed for the entire execution of the algorithm. It will always either be $\frac{1}{cn}$ or $\frac{1}{n\log^c n}$ for constant $c>0$, and it will usually be clear from context which value of $\Bar{p}$ is assumed.
\subsection{Drift analysis and concentration inequalities}

The analysis of evolutionary algorithms relies heavily on drift theorems. These allow to transform statements about the expected one-iteration change of a potential function, a proxy for the function to be optimized, into runtime bounds. For an overview see~\cite{lengler2020drift}. We will need in particular the following negative drift theorem.

\begin{theorem}[Theorem 2 in \cite{oliveto2015improved}]
\label{thm:negative_drift}
    Let $(X_t)_{t \in \mathbb{N}}$ be a stochastic process over some space $S \subseteq \mathbb{R}^+_0$, adapted to a filtration $(\mathcal{F}_t)_{t \in \mathbb{N}}$. Suppose there exist an interval $[a, b] \subseteq \R$ and, possibly depending on $l \vcentcolon= b - a$, a drift bound $\epsilon \vcentcolon= \epsilon(l) > 0$ as well as a scaling factor $r \vcentcolon= r(l)$ such that for all $t \geq 0$ the following three conditions hold:
    \begin{enumerate} 
        \item $\expe{\left[X_{t + 1} - X_t \mid \mathcal{F}_t \, ; \, a < X_t < b\right]} \geq \epsilon$;
        \item $\prob{\left[\lvert X_{t + 1} - X_t \rvert \geq j r \mid \mathcal{F}_t \, ; \, a < X_t\right]} \leq e^{-j}$ for all $j \in \N$;
        \item $1 \leq r^2 \leq \frac{\epsilon l}{132 \log{(r / \epsilon)}}$.
    \end{enumerate}
    Then, for the first hitting time $T^* \vcentcolon= \min{\left\{t \geq 0 : X_t \leq a \mid X_0 \geq b\right\}}$ it holds that $\prob{\left[ T^* \leq e^{\epsilon l / \left(132 r^2 \right)} \right]} = O \left(e^{-\epsilon l / \left(132 r^2 \right)}\right)$.
\end{theorem}

We will further use the variable drift theorem.

\begin{theorem}[Theorem 15 in \cite{lehre2013general}]
    \label{thm:variable_drift}
    Let $(X_t)_{t \in \mathbb{N}}$ be a stochastic process over some state space $S \subseteq \{0\}  \cup \left[ x_\mathrm{min}, x_\mathrm{max}\right]$, adapted to a filtration $(\mathcal{F}_t)_{t \in \mathbb{N}}$, where $x_\mathrm{min} > 0$. Let $h(x) \colon \left[ x_\mathrm{min}, x_\mathrm{max}\right] \to \mathbb{R}^+$ be a monotone increasing function such that $1 / h$ is integrable on $\left[ x_\mathrm{min}, x_\mathrm{max}\right]$ and $\expe{\left[X_t - X_{t+1} \mid \mathcal{F}_t\right]} \geq h(X_t)$ if $X_t \geq x_\mathrm{min}$. Then it holds for the first hitting time $T \vcentcolon= \min{\left\{ t \colon X_t = 0 \right\}}$ that
    \[
    \expe{\left[T \mid X_0 \right]} \leq \frac{x_\mathrm{min}}{h(x_\mathrm{min})} + \int_{x_\mathrm{min}}^{X_0}{\frac{1}{h(x)}\,dx}.
    \]
\end{theorem}

We will also need the next theorem in our proof of Lemma \ref{lemma:bounce_back}.
\begin{theorem}[Theorem 1 in \cite{neumann2010few}]\label{thm:lower_bound_from_drift}
Consider a Markov process $\{X_t\}_{t\ge 0}$ with state space $S$ and a function $g: S \to \mathbb{R}_0^+$. Let $T:=\inf \{t\ge 0: g(X_t) = 0\}$. If there exists $\delta > 0$ such that for every time $t \ge 0$ and every state $X_t$ with $g(X_t)>0$ the condition $\expe{[g(x_t)-g(X_{t+1}) \mid X_t]}\ge \delta$ holds, then
\begin{align}
    \expe{[T \mid X_0]}\ge \frac{g(X_0)}{\delta} \qquad \text{and} \qquad \expe{[T]}\ge \frac{\expe{[g(X_0)]}}{\delta}.
\end{align}
The same statements hold if $T$ is redefined to $T:=\inf \{t \ge 0: g(X_t) \le \kappa\}$ for some $\kappa>0$ and the condition is relaxed to $\expe{[g(X_t) - g(X_{t+1}) \mid X_t \ge \kappa]}\ge \delta$ for all $t \ge 0$ and all $X_t$.
    
\end{theorem}

In addition, we require the following tail bounds for hypergeometric random variables. More explicitly, we consider a random variable $X$ with cumulative distribution function $h_X(M,N,n,i)=\prob{[X=i \mid M,N,n]}= \frac{\binom{M}{i} \binom{N-M}{n-i}}{\binom{N}{n}}$. Here we denote the population size by $N$, the number of success states in the population by $M$, the number of draws by $n$, and the number of observed successes by $i$. Then the following tail estimates hold.

\begin{lemma}[{\cite[Section~5]{skala2013hypergeometric}}]
    Let $X$ be a hypergeometric random variable with cumulative distribution function as described above. Let further $t>0$. Then
    \begin{align*}
        \prob{[X \ge \expe{[X]}+tn}] \le e^{-2t^2n}
        \qquad \text{and} \qquad
        \prob{[X \le \expe{[X]}-tn}] \le e^{-2t^2n}.
    \end{align*}
\end{lemma}

Finally, we make use of the classical Chernoff bound.

\begin{theorem}[Chernoff Bound~{\cite[Section~1.10]{doerr2020probabilistic}}] \label{thm:Chernoff}
Let $X_1, \ldots, X_n$ be in\-dependent random variables taking values in $[0,1]$. Let $X = \sum_{i = 1}^n X_i$ and let $\delta \in[0,1]$. Then 
\begin{align*}
  \Pr[X \ge (1+\delta) \expe{[X]}] \le \exp\bigg(-\frac{\delta^2 \expe{[X]}}{3}\bigg).
\end{align*}
\end{theorem}

\section{Dynamics of the marginal probabilities}
\label{sec:dyn}

We start by analyzing how the \cga behaves on the \dynBV at the level of a single iteration. The first proposition computes the probability that a bit $i$ at which the two offspring differ gets a signal step (in some iteration $t$). This proposition captures the main difference between \DBV and \onemax. For \onemax, a position which differs in the two offspring has probability $1/\max\{\sqrt{V_t},1\}$ to perform a signal step~\cite{witt2019upper}, and performs a random step otherwise. For \DBV, the probability is $1/\max\{V_t,1\}$, so the term $\sqrt{V_t}$ is replaced by $V_t$. Hence, signal steps are more likely for \onemax, and the signal-to-noise ratio of \onemax is larger than for \DBV. This corresponds to the fact that \onemax is a particularly easy function to optimize.

\begin{proposition}
\label{prop:signal_prob_in_1_over_variance}
    Let $K\ge1$ and $\Bar{p}\in(0,\frac{1}{2})$ be arbitrary, and consider the algorithm  $\cga(K, \Bar{p})$ on \DBV and some bit $i \in [n]$ and iteration $t$. For the permutation $\pi_t$ drawn at iteration $t$, we denote by $S_{i, t}$ the event that all bits $i' \in [n]$ that appear before $i$ in the permutation, i.e.\ such that $\pi_t(i') < \pi_t(i)$, are equal in the two offspring. Then it holds that
    \[
    \prob{\left[S_{i, t}\right]} = \Theta \left( \frac{1}{\max\{V_t,1\}} \right).
    \]
\end{proposition}
\begin{proof}
To compute this probability, consider the random variable $D_{i, t}$, which denotes the number of bits $i' \not= i$ that are different in the two offspring. If we denote by $V_t^{\not= i} = \sum_{j\not= i}{p_{j, t} (1 - p_{j, t})}$ the variance of the bits $i' \not= i$, we can observe that $\expe{\left[D_{i, t}\right]} = \sum_{j \not= i}{2 p_{j, t} (1 - p_{j, t})} = 2 V_t^{\not= i}$.

The main observation that will now allow us to find asymptotic bounds for $\prob{\left[S_{i, t}\right]}$ is the following: Once we have fixed the number $D_{i, t} = k \in \{0, 1, \dots, n - 1\}$ of bits $i' \not= i$ that differ in the two offspring, the probability of drawing a permutation $\pi_t$ such that bit $i$ receives the signal is $\frac{1}{k + 1}$. This is because bit $i$ only receives the signal if $\pi_t(i) < \pi_t(i')$ for all bits $i' \not= i$ that differ in the offspring. Since the permutation $\pi_t$ is drawn uniformly at random, the probability that $\pi_t(i)$ assumes the smallest value from a set of $k + 1$ numbers is $\frac{1}{k + 1}$.

To reach the desired conclusion, we will first show that $\prob{\left[S_{i, t}\right]} = \Theta\left( \frac{1}{\max\{V_t^{\not= i}, 1\}} \right)$. Intuitively, this holds because the number $D_{i, t}$ of bits $i'\not= i$ that differ in the offspring is concentrated around its expected value (since it is a sum of independent Bernoulli-distributed random variables), and $\expe{\left[ D_{i, t} \right]} = 2 V_t^{\not= i}$. Formally, this can be shown via Chernoff bounds after conditioning on the value of $D_{i, t}$.

For the lower bound we obtain:
\begin{align*}
    \prob{\left[S_{i, t}\right]} &= \sum_{k = 0}^{n - 1}{\prob{\left[S_{i, t} \mid D_{i, t} = k\right]} \cdot \prob{\left[ D_{i, t} = k \right]}} \\
    &= \sum_{k = 0}^{n - 1}{\frac{1}{k + 1} \cdot \prob{\left[ D_{i, t} = k \right]}} \\
    &\geq \sum_{k = 0}^{\left\lfloor \frac{3}{2} \expe{\left[D_{i, t}\right]} \right\rfloor - 1}{\frac{1}{k + 1} \cdot \prob{\left[ D_{i, t} = k \right]}}.
\end{align*}
Following this truncation of the sum, we can use the fact that for the first factors in the sum is holds that $\frac{1}{k + 1} \geq \frac{1}{\left\lfloor \frac{3}{2} \expe{\left[D_{i, t}\right]} \right\rfloor}$:
\begin{align*}
    \prob{\left[S_{i, t}\right]} &\geq \frac{1}{\left\lfloor \frac{3}{2} \expe{\left[D_{i, t}\right]} \right\rfloor} \cdot \sum_{k = 0}^{\left\lfloor \frac{3}{2} \expe{\left[D_{i, t}\right]} \right\rfloor - 1}{\prob{\left[ D_{i, t} = k \right]}} \\
    &\geq \frac{1}{\frac{3}{2} \expe{\left[D_{i, t}\right]}} \cdot \left(1 - \prob{\left[ D_{i, t} > \left\lfloor \frac{3}{2} \expe{\left[D_{i, t}\right]} \right\rfloor \right]} \right) \\
    &= \frac{1}{\frac{3}{2} \expe{\left[D_{i, t}\right]}} \cdot \left(1 - \prob{\left[ D_{i, t} \geq \left\lfloor \frac{3}{2} \expe{\left[D_{i, t}\right]} \right\rfloor + 1 \right]} \right) \\
    &\geq \frac{1}{\frac{3}{2} \expe{\left[D_{i, t}\right]}} \cdot \left(1 - \prob{\left[ D_{i, t} \geq \frac{3}{2} \expe{\left[D_{i, t}\right]}  \right]} \right).
\end{align*}
The last equality is due to the fact that $D_{i, t}$ only assumes values in $\mathbb{N}$.

Since $D_{i, t}$ is a random variable that is the sum of independent Bernoulli-distributed random variables with $\expe{\left[D_{i, t}\right]} = 2 V_t^{\not= i}$, we can apply the Chernoff bounds to further get:
\begin{equation}
\begin{aligned}
    \label{eq:lower_bound_Sit}
    \prob{\left[S_{i, t}\right]}  &\geq \frac{1}{\frac{3}{2} \expe{\left[D_{i, t}\right]}} \cdot \left(1 - e^{-\frac{1}{12} \expe{\left[D_{i, t}\right]}} \right) \\ 
    & = \frac{1}{3 V_t^{\not= i}} \cdot \left(1 - e^{-\frac{1}{6} V_t^{\not= i}} \right) \\ 
    & = \Omega\left(\frac{1}{\max\{V_t^{\not= i}, 1\}}\right),
\end{aligned}
\end{equation}
where the last step holds for $V_t^{\not= i} \ge 1$ because then $1 - e^{-\tfrac{1}{6} V_t^{\not= i}} \ge 1-e^{-1/6} = \Omega(1)$, and it holds for $V_t^{\not= i} < 1$ because then $1 - e^{-\frac{1}{6} V_t^{\not= i}} = \Theta(V_t^{\not= i})$.

The upper bound can be computed in a very similar manner, again by conditioning on the value of $D_{i, t}$.
\begin{align*}
    \prob{\left[S_{i, t}\right]} &= \sum_{k = 0}^{n - 1}{\prob{\left[S_{i, t} \mid D_{i, t} = k\right]} \cdot \prob{\left[ D_{i, t} = k \right]}} \\
    &= \sum_{k = 0}^{n - 1}{\frac{1}{k + 1} \cdot \prob{\left[ D_{i, t} = k \right]}} \\
    &= \sum_{k = \left\lceil \frac{1}{2} \expe{\left[D_{i, t} \right]} \right\rceil - 1}^{n - 1}{\frac{1}{k + 1} \cdot \prob{\left[ D_{i, t} = k \right]}} + \sum_{k = 0}^{\left\lceil \frac{1}{2} \expe{\left[D_{i, t} \right]} \right\rceil - 2}{\frac{1}{k + 1} \cdot \prob{\left[ D_{i, t} = k \right]}} \\
    &\leq \frac{1}{\left\lceil \frac{1}{2} \expe{\left[D_{i, t} \right]} \right\rceil} \cdot \sum_{k = \left\lceil \frac{1}{2} \expe{\left[D_{i, t} \right]} \right\rceil - 1}^{n - 1}{\prob{\left[ D_{i, t} = k \right]}} + \sum_{k = 0}^{ \left\lceil \frac{1}{2} \expe{\left[D_{i, t} \right]} \right\rceil - 2}{1 \cdot \prob{\left[ D_{i, t} = k \right]}}.
\end{align*}
Since, we are dealing with probabilities of disjoint events, we know that \[\sum_{k = \left\lceil \frac{1}{2} \expe{\left[D_{i, t} \right]} \right\rceil - 1}^{n - 1}{\prob{\left[ D_{i, t} = k \right]}} \leq 1.\] We can employ Chernoff bounds again to obtain:
\begin{align*}
    \prob{\left[S_{i, t}\right]} &\leq \frac{1}{ \frac{1}{2} \expe{\left[D_{i, t} \right]}} + \prob{\left[D_{i, t} \leq  \left\lceil \frac{1}{2} \expe{\left[D_{i, t} \right]} \right\rceil - 2 \right]} \\
    \tag*{(since $D_{i, t}$ only assumes values in $\mathbb{N}$)}
    &\leq \frac{1}{ \frac{1}{2} \expe{\left[D_{i, t} \right]}} + \prob{\left[D_{i, t} \leq \frac{1}{2} \expe{\left[D_{i, t} \right]} \right]} \\
    \tag*{(Chernoff bounds)}
    &\leq \frac{1}{ \frac{1}{2} \expe{\left[D_{i, t} \right]}} + e^{-\frac{1}{8} \expe{\left[ D_{i, t} \right]}} \\
    &= \frac{1}{V_t^{\not= i}} + e^{-\frac{1}{4} V_t^{\not= i}} = O\left(\frac{1}{V_t^{\not= i}}\right).
\end{align*}
Since $\prob{\left[S_{i, t}\right]}$ is a probability, we clearly also have $\prob{\left[S_{i, t}\right]}=O(1)$, and hence $\prob{\left[S_{i, t}\right]} = O\left( \frac{1}{\max\{V_t^{\not= i}, 1\}} \right)$.

We have shown that $\prob{\left[S_{i, t}\right]} = \Theta\left( \frac{1}{\max\{V_t^{\not= i}, 1\}} \right)$, so in order to conclude the proof it remains to show that $\max\{V_t^{\not= i}, 1\} = \Theta(\max\{V_t, 1\})$. Clearly 
$\max\{V_t^{\not= i}, 1\} \le \max\{V_t, 1\}$ since $V_t^{\not= i} \le V_t$. On the other hand, $V_t - V_t^{\not= i} = p_{i,t}(1-p_{i,t}) \le \frac{1}{4}$, and hence
\[
V_t = V_t^{\not= i} + (V_t - V_t^{\not= i}) \le V_t^{\not= i} + \frac{1}{4} = O(\max\{V_t^{\not= i}, 1\}),
\]
which concludes the proof. \qed
\end{proof}

Using Proposition \ref{prop:signal_prob_in_1_over_variance}, one can describe the transition matrix of the marginal probabilities, which then allows to compute the drift of these marginal probabilities. As for Proposition \ref{prop:signal_prob_in_1_over_variance}, the same formulas would hold for \onemax with $V_t$ replaced by $\sqrt{V_t}$.

\begin{proposition}
\label{prop:probability_of_frequency_change}
Let $K\ge1$ and $\Bar{p}\in(0,\frac{1}{2})$ be arbitrary, and consider the algorithm  $\cga(K, \Bar{p})$ on \DBV. Then for all $i\in [n]$ and $t\in\N$ we have $p_{i, t + 1} = \min{\left\{\max{\left\{\Bar{p}, p_{i, t + 1}'\right\}}, 1 - \Bar{p}\right\}}$ where
$$
p_{i, t + 1}' = \begin{cases}
                    p_{i, t}, & \text{with probability } 1 - 2 p_{i, t} (1 - p_{i, t}) \\
                    p_{i, t} + \frac{1}{K}, & \text{with probability } \left(\frac{1}{2} + \Theta\left(\frac{1}{\max\{V_t,1\}}\right)\right) 2 p_{i, t} (1 - p_{i, t}) \\
                    p_{i, t} - \frac{1}{K}, & \text{with probability } \left(\frac{1}{2} - \Theta\left(\frac{1}{\max\{V_t,1\}}\right)\right) 2 p_{i, t} (1 - p_{i, t})
                \end{cases}.
$$
This implies $\expe{\left[p_{i, t + 1} - p_{i, t} \mid p_{i, t}\right]} = \Theta\big(\frac{p_{i, t} (1 - p_{i, t})}{K \cdot \max\{V_t,1\}}\big)$, where the lower bound requires $p_{i, t} < 1 - \Bar{p}$ and the upper bound requires $p_{i, t} > \Bar{p}$.
\end{proposition}

\begin{proof} This proof follows the idea of Lemma 2 in \cite{lengler2021complex}. The main idea is simply to use that the two offspring differ at bit $i$ in iteration $t$ with probability $2 p_{i, t} (1 - p_{i, t})$, and that this probability does not change if we condition on $S_{i,t}$ since this event only influences the offspring value at some bits $i'$ different from $i$.

First note that $p_{i, t + 1}' \not= p_{i, t}$ if and only if the two bits at position $i$ in the offspring are sampled differently. This implies that the event $p_{i, t + 1}' = p_{i, t}$ happens with probability $1 - p_{i, t} (1 - p_{i, t}) - (1 - p_{i, t}) p_{i, t} = 1 - 2 p_{i, t} (1 - p_{i, t})$. We will now derive asymptotically tight bounds for the probability that $p_{i, t + 1}' = p_{i, t} + \frac{1}{K}$, as the bounds for $p_{i, t + 1}' = p_{i, t} - \frac{1}{K}$ follow analogously by symmetry.

Note that during one iteration, at most one bit in the offspring takes a signal step, and all other bits take random steps. To be more precise, there exists a signal step if and only if the two offspring have a bit in which they differ.

Consider therefore an arbitrary iteration $t$. Let $S_{i, t}$ be the event defined in Proposition \ref{prop:signal_prob_in_1_over_variance}, and let $\bar{S}_{i, t}$ be its complement, namely that bit $i$ did not receive the signal for sure (which implies that bit $i$ will do a random step). Conditioning on these events, we can rewrite the probability of $p_{i, t}' = p_{i, t} + \frac{1}{K}$ as:
\begin{equation}
\begin{aligned}
\label{eq:split_probability_condition}
\prob{\left[ p_{i, t + 1}' = p_{i, t} + \frac{1}{K} \right]} &= \prob{\left[p_{i, t + 1}' = p_{i, t}  + \frac{1}{K} \Bigm| S_{i, t}\right]} \cdot \prob{\left[ S_{i, t} \right]} \\
& + \prob{\left[p_{i, t + 1}' = p_{i, t} + \frac{1}{K} \Bigm| \bar{S}_{i, t} \right]} \cdot \prob{\left[ \bar{S}_{i, t} \right]}.
\end{aligned}
\end{equation}
Let us now analyze the factors in the formula above. If bit $i$ has the chance of receiving the signal, then it has the possibility to determine whether the update with respect to the first or second offspring. Therefore, in order to observe an increase in its frequency, it only needs to happen that the two offspring differ at bit $i$. Thus, the first conditional probability is 
\begin{equation}
\label{eq:split_part_1}
\prob{\left[p_{i, t + 1}' = p_{i, t} + \frac{1}{K} \Bigm| S_{i, t}\right]} = 2 p_{i, t} (1 - p_{i, t}).
\end{equation}

On the other hand, if bit $\bar{S}_{i, t}$ takes place, then bit $i$ does not contribute to the decision whether to update with respect to the first or second offspring. Thus, we observe an increase in its marginal frequency only if bit $i$ is set to 1 in the fitter individual and to 0 in the other. This leads to the conditional probability 
\begin{equation}
\label{eq:split_part_2}
\prob{\left[p_{i, t + 1}' = p_{i, t} + \frac{1}{K} \Bigm| \bar{S}_{i, t}\right]} = p_{i, t} (1 - p_{i, t}).
\end{equation}

As already mentioned in Proposition \ref{prop:signal_prob_in_1_over_variance}, the probability of the event $S_{i, t}$ taking place is nothing but the probability that all the bits $i' \not= i$ for which $\pi_t(i') < \pi_t(i)$ are equal in the two offspring, and it holds that
\[
    \prob{\left[S_{i, t}\right]} = \Theta \left( \frac{1}{\max\{V_t,1\}} \right).
\]
Combining this result with equations \eqref{eq:split_probability_condition}, \eqref{eq:split_part_1} and \eqref{eq:split_part_2} yields:
\begin{equation}
\begin{aligned}
\label{eq:prob_of_step_up}
\prob{\left[ p_{i, t + 1}' = p_{i, t} + \frac{1}{K} \right]} &= 2 p_{i, t} (1 - p_{i, t}) \cdot \prob{\left[S_{i, t}\right]} + p_{i, t} (1 - p_{i, t}) \cdot \left(1 - \prob{\left[S_{i, t}\right]} \right) \\
&= p_{i, t} (1 - p_{i, t}) \cdot \left( 1 + \prob{\left[S_{i, t}\right]} \right) \\
&= p_{i, t} (1 - p_{i, t}) \cdot \left( 1 + \Theta \left( \frac{1}{\max\{V_t,1\}} \right) \right)
\end{aligned}
\end{equation}
The final statement of the lemma on the expectation follows easily from the aforementioned probability bounds and the way the \cga caps the probabilities at the boundaries.
\qed
\end{proof}

\section{Lower bound on the runtime}
\label{sec:lower_bound}

In this section, we prove lower bounds for the runtime of the compact Genetic Algorithm on \dynBV when $K$ is polynomial in the number of bits. The idea is to bound the number of signal steps that a given bit makes over a certain number of iterations, and use this bound to show that a linear number of frequencies stay a constant distance away from the boundaries for $\Omega(K \cdot\min\{K,n\})$ iterations. As long as this is the case, the probability to sample the optimum in any given iteration is exponentially small, and hence a union bound over the iterations gives us a high probability lower bound on the runtime. We further prove Theorem~\ref{thm:geneticdrift}, which states that genetic drift occurs whenever $K=O(n)$.

We start by upper bounding the probability that a fixed frequency gets a signal step in a given iteration, under the condition that enough bits have marginal probabilities far away from the boundaries.

\begin{corollary}
\label{cor:probability_of_signal_step}
Let $K\ge1$ and $\Bar{p}\in(0,\frac{1}{2})$ be arbitrary, and consider the algorithm  $\cga(K, \Bar{p})$ on \DBV. 
Assume that at iteration $t$ there are at least $\gamma n $ bits whose marginal probabilities are within $[\tfrac{1}{6}, \tfrac{5}{6}]$, for some constant $\gamma > 0$. Then the probability of having a signal step on any fixed bit is $O\left(\frac{1}{n}\right)$.
\end{corollary}

\begin{proof}
Because there are at least $\gamma n$ bits whose marginal probabilities are within the interval $[\tfrac{1}{6}, \tfrac{5}{6}]$, and for each of these we can lower bound the product $p_{i,t}(1-p_{i,t}) \ge \frac{1}{6}\cdot\frac{5}{6} = \frac{5}{36}$, we have
\begin{align*}
    V_t = \sum_{i = 1}^n{p_{i, t} (1 - p_{i, t})} \ge \frac{5\gamma}{36}n = \Omega(n).
\end{align*}
Now fix a bit $i$ at iteration $t$, and recall the event $S_{i,t}$ from Proposition \ref{prop:signal_prob_in_1_over_variance}. This event is a necessary condition for bit $i$ to have a signal step in iteration $t$, and hence (using Proposition \ref{prop:signal_prob_in_1_over_variance} and the bound on the sampling variance) the probability that this happens is at most
\begin{align*}
    \prob{[S_{i,t}]} = \Theta\left(\frac{1}{\max\{V_t, 1\}}\right) = O\left(\frac{1}{n}\right).
\end{align*}
\qed

\end{proof}

Next, we prove a lemma that guarantees that the displacement of the marginal probabilities caused by $O(K^2)$ random steps is bounded in absolute value by $\tfrac{1}{6}$ for a linear number of bits. The approach is similar to the proof of Lemma 13 in \cite{sudholt2019choice}.

\begin{lemma}
\label{lemma:bits_stay_within_range}
Let $K\ge1$ and $\Bar{p}\in(0,\frac{1}{2})$ be arbitrary, and consider the algorithm  $\cga(K, \Bar{p})$ on \DBV.
Consider a fixed frequency $i\in[n]$ and let $t \leq \alpha K^2$, where $\alpha > 0$ is a small enough constant. With probability $\Theta(1)$, the first $t$ random steps of frequency $i$ lead to a total change of the bit's marginal probability that is within $\left[ -\frac{1}{6}, \frac{1}{6} \right]$.

Moreover, for a small enough constant $\gamma >0$, the probability that the above holds for less than $\gamma n$ bits among the first $\frac{n}{2}$ bits is $2^{-\Omega(n)}$, regardless of the decisions made on the last $\frac{n}{2}$ bits.
\end{lemma}

\begin{proof}
Fix an arbitrary bit $i$. Define the random variables $Z_k = K \cdot (p_{i, k + 1} - p_{i, k})$ to be the scaled change of the bit's marginal probability at an iteration $k$ where bit $i$ does a random step. Note that these variables take values in $\{-1,0, 1\}$. The proof consists in bounding the maximum value that the (partial) sums of the $Z_k$'s can take over the first $t$ random steps, using a "maximal" version of Chernoff's bound.

We thus define $Y_j \coloneqq \sum_{k = 1}^j{Z_k}$ as the total progress in the first $j \in \{1, 2, \dots, t\}$ random steps. Due to the symmetry and the fact that we start with the marginal probability equal to $\frac{1}{2}$, we can conclude $\expe{\left[ Y_j \right]} = 0$. Since the change in the probabilities is bounded by $\frac{1}{K}$, we need $Y_j \geq \frac{K}{6}$ in order to have a total change that exceeds $\frac{1}{6}$. Using the Chernoff-Hoeffding bounds from Theorems 1.11 and 1.13 in \cite{doerr2011analyzing}, we can bound the probability of this event happening during the first $t$ random steps:
\[\prob \left[ \max_{j \in \{1, \dots, t\}}{Y_j} \geq \frac{K}{6} \right] \leq \exp\left( -2 \cdot \frac{\left(\frac{K}{6}\right)^2}{\sum_{j = 1}^t{2^2}} \right) = \exp\left( -\frac{K^2}{72t} \right) \leq \exp \left( - \frac{1}{72\alpha} \right).\]
Leveraging the symmetry of this process, we can obtain the same bound for the probability of the total change going below $-\frac{1}{6}$ 
\[\prob\left[ \min_{j \in \{1, \dots, t\}}{Y_j} \leq -\frac{K}{6} \right] = \prob\left[ \max_{j \in \{1, \dots, t\}}{-Y_j} \geq\frac{K}{6} \right] \leq \exp{\left( - \frac{1}{72 \alpha}\right)}.\]
Now using the union bound, we can bound the probability that the total change stays within $\left[-\frac{1}{6}, \frac{1}{6} \right]$ the first $t$ random steps, denoted by the $P^{\leq t}$:
\begin{align*}
    P^{\leq t} &= 1 - \prob\left[\left( \max_{j \in \{1, \dots, t\}}{Y_j} > \frac{K}{6} \right) \cup \left( \min_{j \in \{1, \dots, t\}}{Y_j} < -\frac{K}{6} \right)\right] \\
    &\geq 1 - \prob\left[ \max_{j \in \{1, \dots, t\}}{Y_j} > \frac{K}{6} \right] - \prob\left[ \min_{j \in \{1, \dots, t\}}{Y_j} < -\frac{K}{6}\right] \\
    &\geq 1 - \prob\left[ \max_{j \in \{1, \dots, t\}}{Y_j} \geq \frac{K}{6} \right] - \prob\left[ \min_{j \in \{1, \dots, t\}}{Y_j} \leq -\frac{K}{6}\right] \\
    &\geq 1 - 2 \exp{\left(-\frac{1}{72 \alpha} \right)}.
\end{align*}
Choosing $\alpha < \frac{1}{72 \ln{4}}$ we get $P^{\leq t} \ge \frac{1}{2} = \Omega(1)$, which proves the first part of the lemma.

Since by definition the random steps of different frequencies behave independently of each other, the second part of the lemma follows directly using Chernoff bounds.
\qed
\end{proof}

In the next lemma, we put things together and show that there is a constant fraction of bits whose marginal probabilities stay bounded away from the boundaries (and hence receive few signal steps). The proof structure follows that of Lemma 15 in \cite{sudholt2019choice}, and the idea is to bound the accumulated effect of signal steps using Chernoff bounds.

\begin{lemma}
\label{lemma:slow_progress}
Let $K \ge 1$ and $\Bar{p}\in(0,\frac{1}{2})$ be arbitrary, and consider the algorithm  $\cga(K, \Bar{p})$ on \DBV. There exist constants $\alpha, \gamma > 0$, such that the following holds with high probability, regardless of the last $\frac{n}{2}$ bits (i.e., an adversary may choose the value of those bits in the offspring). There is a subset $S$ of $\gamma n$ bits among the first $\frac{n}{2}$ bits such that during the first $t := \alpha K^2$ iterations:
\begin{enumerate}[label=\roman*)]
    \item the marginal probabilities of all bits in $S$ always lie in the interval $\left[\frac{1}{6}, \frac{5}{6}\right]$;
    \item the total number of signal steps for each bit in $S$ is bounded by $\frac{K}{6}$, leading to a displacement of at most $\frac{1}{6}$. 
\end{enumerate}
\end{lemma}

\begin{proof}
Note that the first part holds at initialization. In order to prove the first part for all iterations, we will look at the probability that the property we want to prove gets violated. We will show that in each iteration, the property gets violated only if an event of probability exponentially small in $K$ occurs. Taking the union bound over $t$ iterations then shows that the property must hold throughout all of the $t$ iterations.

If $\gamma >0$ is a sufficiently small constant, it follows from the second part of Lemma \ref{lemma:bits_stay_within_range} that with probability $1 - 2^{-\Omega(n)}$ at least $2\gamma n$ bits among the first $\tfrac{n}{2}$ ones experience a displacement from the random steps that lies within $[-\tfrac{1}{6}, \tfrac{1}{6}]$ during the first $t \leq \alpha K^2$ iterations. For the remainder of the proof, we will assume that this event takes place, and we will consider $S'$ to be a set of $2\gamma n$ such bits.

To complete the proof, we distinguish two cases. Either $K \le \log^2 n$. Then $t=o(n)$. Since at most one bit per iteration takes a signal step, the total number of bits in $S'$ which ever take a signal step is at most $t \le \gamma n$, for $n$ sufficiently large. Hence, at least $\gamma n$ bits in $S'$ never take a signal step, which completes this case.

For $K\ge \log^2 n$ we will show that for all bits in the set $S'$, with high probability, the total displacement caused by the signal steps is at most $\frac{1}{6}$, so we may take $S:=S'$. According to Corollary \ref{cor:probability_of_signal_step}, for every bit in $S'$, the probability of it taking a signal step in an iteration is bounded by $\frac{c}{n}$, where $c > 0$ is the constant from the lemma's $O$-notation. As long as the conditions for Corollary \ref{cor:probability_of_signal_step} hold, the expected number of signal steps of a fixed bit in $t$ iterations will be at most $t \cdot \frac{c}{n} \leq \alpha K \cdot \min{\left\{K, n\right\}} \cdot \frac{c}{n} \leq \alpha c K$.

We know that each signal step changes the marginal probability of a bit by $\frac{1}{K}$, so a necessary condition for increasing a bit's marginal probability by an additional $\frac{1}{6}$ would be to take at least $\frac{K}{6}$ signal steps during the $t$ iterations. By choosing the constant $\alpha > 0$ so small enough that $\alpha c K \leq \frac{1}{2} \cdot \frac{K}{6}$, Chernoff's bound ensures that the probability of getting at least $\frac{K}{6}$ signal steps in the first $t$ iterations is at most $e^{- \Omega(K)}$. 
Thus, to violate the first property, an event of probability $e^{- \Omega(K)}$ has to take place for any bit in $S'$ and any length of time, otherwise all properties hold.

Finally, we can take a union bound over all $2\gamma n$ bits in the set $S'$, as well as all $t$ iterations, and obtain that the desired property gets violated with probability $t \cdot 2\gamma n \cdot e^{- \Omega(K)} \le \alpha K^2 \cdot 2 \gamma n \cdot e^{-\Omega(K)}$. Since $K \ge \log^2 n$, this failure probability is $o(1)$, and since the other property is a direct implication of Lemma \ref{lemma:bits_stay_within_range}, the proof is complete.

\qed
\end{proof}

The main result of this section is then a direct consequence of the above lemma.

\begin{theorem}\label{thm:lower_bound}
Let $K =O(\mathrm{poly}(n))$ and $\Bar{p}\in (0,\frac{1}{2})$ be arbitrary, and consider the algorithm  $\cga(K, \Bar{p})$ on \DBV. Then with high probability the optimum is not sampled during the first $\Omega(K\cdot\min\{K,n\})$ iterations.
\end{theorem}

\begin{proof}
Let $\alpha>0$ be the constant from Lemma \ref{lemma:slow_progress}. By this lemma, the probability of sampling the optimum in one given iteration among the first $\alpha K \cdot \min\{K, n\}$ iterations is at most $2\cdot\left( \frac{5}{6} \right)^{\gamma n} = e^{- \Omega(n)}$, since all $\gamma n$ bits with marginal probability in the interval $[\frac{1}{6}, \frac{5}{6}]$ have to be set to 1 (and there are two offspring in each iteration). Taking a union bound over these $\Omega(K\cdot\min\{K,n\})$ iterations and adding the error probability from Lemma \ref{lemma:slow_progress} concludes the proof. 
\qed
\end{proof}

We are also in a position to show the existence of genetic drift. For the convenience of the reader, we restate Theorem~\ref{thm:geneticdrift} below before finishing its proof.

\begin{theorem}
\label{thm:genetic_drift}
Let $K =\omega(1)$ and $\Bar{p}\in (0,\frac{1}{2})$ be arbitrary, and consider the algorithm  $\cga(K, \Bar{p})$ on \DBV.
For every $\rho>0$ and $\beta\in (\Bar{p},\frac{1}{2})$ there is $\delta>0$ such that the following holds. Assuming that $K \le \rho n$, then with high probability at least $\delta n$ frequencies drop below $\beta$ during optimization.
\end{theorem}

\begin{proof}
We begin by establishing that the sample variance $V_t$ remains of linear order for our period of interest. This is again due to a linear number of bits which stay in middle interval around their initialization values. Concretely, by Lemma \ref{lemma:slow_progress}, for some small constants $\alpha, \gamma>0$, with high probability there exists a subset of $\gamma n$ bits from the first $\frac{n}{2}$ such that the marginal probabilities of these bits remain in the interval $\left[\frac{1}{6}, \frac{5}{6}\right]$ for the first $\alpha K^2$ iterations. In particular, this implies that $V_t = \Omega(n)$ for all $t \leq \alpha K^2$.

Consider now the other $\frac{n}{2}$ bits, denoting them by $D_0 \subseteq [n]$. By Corollary \ref{cor:probability_of_signal_step}, each of them has probability of $O(\frac{1}{n})$ to take a signal step during these $\alpha K^2$ iterations. This implies that the expected number of signal steps taken by one of these bits during this timespan is $\alpha K^2 \cdot O(\frac{1}{n}) \leq O(K)$, where the inequality holds since $K \leq \rho n$. Hence the number of signal steps taken by a given bit in $D_0$ is stochastically dominated by a binomial random variable $\mathrm{Bin}\left(\alpha K^2, O(\frac{1}{n}) \right)$.

Chernoff bounds therefore imply the existence of a constant $\eta > 0$ such that a bit in $D_0$ takes at most $\eta K$ signal steps during the first $\alpha K^2$ iterations with constant probability. 

Let us now focus on the subset $D_1 \subseteq D_0$ of bits that have taken at most $\eta K$ signal steps throughout the first $\alpha K^2$ iterations. Since the probability that a bit in $D_0$ is also in $D_1$ is constant, we get by a similar argument (involving Chernoff bounds and stochastic dominance), that $\lvert D_1 \rvert = \Omega(n)$ with high probability.
Note that all the bits in $D_1$ make at most $\eta K$ signal steps during the $\alpha K^2$ iterations, and as long as their marginal probabilities are in the interval $[\beta, 1-\beta]$, the probability that a given iteration is a random step is lower bounded by a constant. Therefore, since $K=\omega(1)$ (and hence $K=o(K^2)$), we can again use Chernoff bounds to deduce that the number of random steps they make is at least $\eps K^2$ with constant probability for some small constant $\eps >0$. Now let $D_2 \subseteq D_1$ denote the set of bits from $D_1$ that have taken at least $\eps K^2$ random steps throughout the first $\alpha K^2$ iterations. By Chernoff, we know that $|D_2| = \Omega(n)$ with high probability.

Consider now the effect these random steps will have on the marginal probability of a bit in $D_2$. To this effect, we will inspect some properties of unbiased random walks with step size 1.
Let $X_N$ denote the endpoint of such an unbiased random walk with $N$ steps. Then, by the Central Limit Theorem, for every $c \in \R$ it holds that
\[
    \prob{[X_N \leq c \sqrt{N}]} \to \Phi(c) \text{, as } N \to \infty,
\]
where $\Phi$ is the cumulative distribution function of the standard normal distribution. Let us instantiate $N \coloneq \eps K^2 $ and $c \coloneq \frac{-1 / 2 + \beta - \eta}{\sqrt{\eps}}$. Observe that $N \to \infty$ is equivalent to $K \to \infty$, hence for $K$ large enough
\[
    \prob{[X_N \leq (-\tfrac{1}{2}+\beta-\eta) K]} = \prob{[X_N \leq c \sqrt{N}]} \ge \frac{\Phi(c)}{2}.
\]

We return now to our setting of $\eps K^2$ random steps of size $\frac{1}{K}$. Until this random walk reaches distance $\frac{1}{2} - \beta + \eta$ from its starting point, it behaves exactly like a rescaled version of the unbiased random walk. 

Moreover, by symmetry both endpoints (namely $\frac{1}{2} - \beta + \eta$ and $-\frac{1}{2} + \beta - \eta$) have the same probability of having been reached first, i.e.\ the probability to end at $-\frac{1}{2} + \beta - \eta$ is $\frac{1}{2}$.

Hence, the probability that these $\eps K^2$ random steps will cause a total (negative) displacement of at least $\frac{1}{2} - \beta + \eta$ from the starting point is at least $\frac{\Phi(c)}{4}$. 

Thus, for any bit in $D_2$, we have that the cumulative effect of the random and signal steps results in a marginal probability of at most $\frac{1}{2} + \eta -(\frac{1}{2} - \beta + \eta) = \beta$ with constant probability. Consider $D_3 \subseteq D_2$ to be the set of bits for which this holds true. We deduce that $|D_3| \ge \delta n$ with high probability using Chernoff bounds, for some constant $\delta>0$ chosen small enough, which concludes the proof.

\qed
\end{proof}

\section{Upper bound on the runtime for the conservative regime}
\label{sec:nlogn}

This section is dedicated to showing an  upper bound on the runtime for population sizes $K=\Omega(n \log n)$ that matches Theorem \ref{thm:lower_bound}. To facilitate the proof, we slightly adjust the boundaries by setting $\Bar{p}=\frac{1}{cn}$ for $c$ large enough. Precisely, our goal will be to show the following result.

\begin{theorem}
    \label{thm:upper_bound_runtime}
    Let $K=\mathrm{poly}(n)$ and let $c,c'>0$ be sufficiently large constants. If $K \geq c' \cdot n \log{n}$, then the expected optimization time of the algorithm $\cga(K, \frac{1}{cn})$ on \DBV is $O(K n)$.
\end{theorem}

We begin with a proposition that is the analogue of Lemma 4 in~\cite{sudholt2019choice}. Intuitively, it states that for $K= \Omega(n \log{n})$, the marginal probability of any bit above at least some constant threshold $\beta$ will not drop below a fixed constant $0<\alpha<\beta$ in the next polynomially many iterations. Even more succinctly, even moderately high marginal probabilities do not drop by much for a while.

\begin{proposition}
\label{prop:bound_stay_between_constants}
Let $\Bar{p}\in(0,\frac{1}{2})$ be arbitrary and let $K\ge1$, and consider the algorithm  $\cga(K, \Bar{p})$ on \DBV.
Let $\Bar{p} < \alpha < \beta < 1 - \Bar{p}$ and $\gamma > 0$ be constants. There exists a constant $c' > 0$ (possibly depending on $\alpha$, $\beta$, and $\gamma$) such that for a specific bit the following holds: If the bit has marginal probability at least $\beta$ and $K \geq c' \cdot n \log{n}$, then the probability that during the following $n^\gamma$ iterations the marginal probability decreases below $\alpha$ is at most $O(n^{- \gamma})$.
\end{proposition}

\begin{proof}
The goal of this proof is to apply the negative drift theorem, Theorem \ref{thm:negative_drift}, so it will be rather technical, with the bulk of the task consisting of setting the parameters and verifying the three conditions of the negative drift theorem. To this extent, we first fix an arbitrary frequency $i$ and consider the scaled stochastic process $X_t \vcentcolon= K p_{i, t}$.

To apply Theorem \ref{thm:negative_drift}, we choose the interval bounds $a \vcentcolon= K \alpha$ and $b \vcentcolon= K \beta$, which yields $l \vcentcolon= b - a = K (\beta - \alpha)$. To establish the first condition of the theorem, we scale the result from Proposition \ref{prop:probability_of_frequency_change} by $K$ and obtain the following bound on the drift (thus, also fixing $\epsilon \vcentcolon= 4 \xi \alpha (1 - \beta) / n$ in the meantime 
\begin{align*}
    \expe{\left[X_{t + 1} - X_t \mid \mathcal{F}_t \, ; \, a < X_t < b\right]} &\geq K \cdot \xi \cdot \frac{p_{i, t} (1 - p_{i, t})}{K V_t} \\
    &\geq \xi \frac{4 \alpha (1 - \beta)}{n},
\end{align*}
where $\xi > 0$ is the constant from the $\Theta$-notation in Proposition \ref{prop:probability_of_frequency_change}, and the final bound was obtained by observing that $V_t = \sum_{j = 1}^n{p_{j, t} (1 - p_{j, t})} \leq \sum_{j = 1}^n{\frac{1}{4}} = \frac{n}{4}$.

The second condition of the theorem can be trivially fulfilled by choosing $r \vcentcolon= 2$. We note that for $j = 0$, the condition holds because $e^{-0}=1$ by definition and the left-hand side is a probability, hence again by definition upper-bounded by 1 - and for $j \geq 1$ we have that $\prob{\left[ \lvert X_{t + 1} - X_t \rvert \geq 2j \right]} = 0 \leq e^{-j}$, because $\lvert X_{t + 1} - X_t \rvert \leq 1$ in the scaled process, as the marginal frequencies change by at most $\frac{1}{K}$.

We verify the last condition of the theorem using the fact that $K \geq c \cdot n \log{n}$:
\begin{align*}
    \frac{\epsilon l}{132 \log{(r / \epsilon)}} &= \frac{4 \xi \alpha (1 - \beta)}{n} \cdot K (\beta - \alpha) \cdot \frac{1}{132} \left( \log{n} - \log{(2 \xi \alpha (1 - \beta))}\right)^{-1} \\
    & \geq \frac{\xi \alpha (1 - \beta) (\beta - \alpha)}{33} \cdot c \cdot \frac{\log{n}}{\log{n} - \log{(2 \xi \alpha (1 - \beta))}}.
\end{align*}
As all $\xi, \alpha, \beta > 0$ are constants and $(1-\beta)(\beta-\alpha)$ is positive since $1>\beta>\alpha$, the term above is positive for large enough $n$. From the limiting value $\lim_{n \to \infty}{\frac{\log{n}}{\log{n} - \log{(2 \xi \alpha (1 - \beta))}}} = 1$, we can conclude that there exists a constant $c > 0$ for which the term above is greater or equal to $4 = r^2$.

Furthermore, note that the constant $c > 0$ in $K \geq c \cdot n \log{n}$ can further be increased to ensure that the following inequality holds
\begin{align*}
    \frac{\epsilon l}{132 r^2} &= \frac{4 \xi \alpha (1 - \beta) \cdot K (\beta - \alpha)}{132 \cdot 4 \cdot n} \\
    &\geq \frac{\xi \alpha (1 - \beta) (\beta - \alpha)}{132} \cdot c \cdot \log{n} \\
    &\geq \gamma \ln{n},
\end{align*}
for an arbitrarily large constant $\gamma$. This is possible because all other quantities, that is, $\xi, \beta, \alpha$ do not depend on $n$ and have been fixed beforehand. \\ Now we are ready to conclude. Since, by assumption, we started with $X_0 \geq b$, we can establish via the negative drift theorem that $\prob{\left[ T \leq n^\gamma \right]} = O \left( n^{- \gamma} \right)$, where $T$ is the first time when the fixed marginal frequency drops below $\alpha$.
\qed
\end{proof}

The next lemma, which is modeled after Lemma 2 in \cite{neumann2010few}, gives an upper bound on the time until the marginal probability reaches any threshold $\tau$, irrespective of where the marginal probability starts. In the proof of the main theorem of this section, this lemma will allow us to upper-bound the "recovery time" needed for bits whose marginal probabilities travel to the lower boundary.

\begin{lemma}
    \label{lemma:bounce_back}
    Let $c>0$ be a constant and let $K\ge 1$. Let $\tau \in \left[ \frac{1}{c n}, 1 - \frac{1}{c n} \right]$ and consider some fixed bit $i$ in the  $\cga(K, \frac{1}{cn})$ on \DBV with an arbitrary value for the initial marginal probability. Then the expected time until the marginal probability $p_i$ of this bit reaches at least $p_i=\tau$ is $O(K n^2)$.
\end{lemma}

\begin{proof}
    Let us inspect the expected change in probability for the marginal frequency of bit $i$ at an arbitrary iteration $t$ for which $p_{i, t} < \tau$. From Proposition \ref{prop:probability_of_frequency_change} we can fix $\xi > 0$ as the constant in the $\Theta$-notation and write:
    \begin{align*}
        \expe{\left[p_{i, t + 1} \mid p_{i, t}\right]} &\geq p_{i, t} +  \xi \cdot \frac{p_{i, t} (1 - p_{i, t})}{K V_t} \\
        &\geq p_{i, t} + \xi \cdot \frac{1}{cn} \left(1 - \frac{1}{cn} \right) \cdot \frac{1}{K} \cdot \frac{4}{n} \\
        &= p_{i, t} + \Omega \left( \frac{1}{K n^2} \right),
    \end{align*}
    where the second inequality stems from the fact that the marginal frequencies are capped at $\frac{1}{cn}$ and $1 - \frac{1}{cn}$ at the boundaries, and $V_t = \sum_{j = 1}^n{p_{j, t} (1 - p_{j, t})} \leq \sum_{j = 1}^n{\frac{1}{4}} = \frac{n}{4}$.

    The claim now follows by applying Theorem \ref{thm:lower_bound_from_drift} with drift bound $\delta = \Omega \left( \frac{1}{K n^2} \right)$, threshold $\kappa = 1 - \tau$, and distance function $g$ with $g(x) = 1 - x$.
    \qed
\end{proof}

We are now ready to prove the main theorem of this section. Its proof follows the structure of the proof of Theorem 2 in \cite{sudholt2019choice}.

\begin{proof}[of Theorem \ref{thm:upper_bound_runtime}]
    Recall that we assumed without loss of generality that $\frac{1}{K}$ ``divides'' $\frac{1}{2} - \frac{1}{c n}$, which implies that the marginal frequencies are restricted to the set \[\left\{ \frac{1}{c n}, \frac{1}{c n} + \frac{1}{K}, \dots, \frac{1}{2}, \dots, 1 - \frac{1}{c n} - \frac{1}{K}, 1 - \frac{1}{c n} \right\}.\]

    The structure of this proof will follow the one of Theorem 2 in \cite{sudholt2019choice}. We will show that starting with a state where all frequencies are at least $\frac{1}{2}$ at the same time, with probability $\Omega(1)$, after $O(K n)$ iterations either the global optimum has been found, which we call \emph{success}, or at least one of the frequencies has dropped below some arbitrary constant $\eta < \frac{1}{2}$, which we call \emph{failure}. We do not instantiate the constant $\eta$. Concretely, if a failure occurred, then we will prove that one can recover from it  - without too big of an overhead in expectation. Recovering means that we return to a state where all marginal frequencies are at least $\frac{1}{2}$. The expected time until either a success or a failure takes place is $O(K n)$. Then, if no failure occurs during $O(K n)$ iterations, we show that with probability $\Omega(1)$ the optimum has been found.

    Firstly, let us prove that recovering from a failure is not too costly. Choose a constant $\gamma > 0$ such that $n^\gamma \geq K \cdot n^5$ (such a constant exists since $K = \mathrm{poly}(n)$). By Proposition \ref{prop:bound_stay_between_constants}, instantiated with $\alpha = \eta < \frac{1}{2}$ and $\beta = \frac{1}{2}$, the probability of a failure in $n^\gamma$ iterations is at most $O(n^{- \gamma})$, provided the constant $c$ in $K \geq c \cdot n \log{n}$ is large enough. From Lemma \ref{lemma:bounce_back} we can conclude that the expected time for the  marginal probability of a specific bit to reach the upper boundary is always bounded by $O(K n^2)$, regardless of the initial state. We note that the expected time until all probabilities have reached the upper boundary at least once is in $O(K n^3 \log{n})$, as it is stochastically dominated by a simple modification of the coupon collector problem: We split the time into phases of length $\kappa \cdot Kn^2$ for some constant $\kappa > 0$ that is large enough. In each of these phases, by Markov's inequality and Lemma \ref{lemma:bounce_back}, each marginal probability travels to the upper boundary with probability at least $p > 0$, where $p$ is a constant. Therefore, by the coupon collector problem, we need in expectation $O(n \log{n})$ such phases for all bits to reach the upper boundary at least once. Once a bit reaches the upper boundary, we apply Proposition \ref{prop:bound_stay_between_constants} with $\alpha = \frac{1}{2}$ and $\beta = \frac{2}{3}$ to get that the probability of a marginal frequency decreasing below $\frac{1}{2}$ in time $n^\gamma$ is at most $O(n^{- \gamma})$. By the union bound, the probability that there is a bit for which this happens is at most $O(n^{- \gamma + 1})$. If this does not happen, all bits assume values of at least $\frac{1}{2}$ simultaneously, and we can say that we successfully recovered from a failure.

    As the probability of a bit falling below $\frac{1}{2}$ during the recovery phase is at most $n^{- \gamma + 1}$, the expected number of restarts of this recovery phase is $O(n^{- \gamma + 1})$. Therefore, considering the expected time until all bits recover to values that are at least $\frac{1}{2}$ only adds an additional term of $O(n^{- \gamma + 1} \cdot K n^3 \log{n}) = o(1)$ to the expectation (recall that $n^\gamma \geq K n^5$).

    All that remains to show now is that after $O(K n)$ iterations without failure the probability of having found the optimum is in $\Omega(1)$. To this end we will consider a suitable potential function, and apply the variable drift theorem (see \cite{lengler2020drift}). We will apply the variable drift theorem as long as the sampling variance stays above an arbitrary constant threshold. Since the variance will be a lower bound on the potential, this will imply that the potential also assumes values above a certain threshold. In the subsequent, we will show that once the variance drops below a certain constant threshold, the potential will as well (as we are in the case where all marginal probabilities are at least $\eta$), and we will sample the optimum with at least constant probability.

    The potential function we consider is $\varphi_t \vcentcolon= n - \frac{1}{c} - \sum_{i = 1}^n{p_{i, t}}$, which measures the distance to the ideal setting where all frequencies lie at the upper boundary. Recall here that we are inspecting the modified version of the \cga where the upper boundary is at $1 - \frac{1}{c n}$. 

    Assume for now that $V_t \geq 1$ holds for the sampling variance $V_t$. This assumption is helpful, as it will allow us to get asymptotic bounds on the drift whose constants do not depend on $c$ (see Proposition \ref{prop:probability_of_frequency_change}).
    
    We will now inspect the drift of a fixed bit $i \in \{1, \dots, n\}$ by distinguishing two cases based on whether the frequency is at the upper boundary or not:
    \begin{enumerate}[label=\Roman*.]
        \item $p_{i, t} > 1 - \frac{1}{c n} - \frac{1}{K}$: By our assumption that $\frac{1}{K}$ ``divides'' $\frac{1}{2} - \frac{1}{c n}$, we conclude that $p_{i, t} = 1 - \frac{1}{c n}$, therefore the frequency of bit $i$ can only decrease. A decrease by $\frac{1}{K}$ happens with probability at most $\frac{1}{c n}$: We have to sample a $0$ on bit $i$ in the ``losing'' offspring, and a $1$ in the ``winning'' one. \footnote{I.e. the less fit respectively fitter offspring.} This happens with probability $p_{i, t} (1 - p_{i, t}) \leq \frac{1}{c n}$. Thus:
        \begin{align*}
            \expe{\left[ 1 - \frac{1}{c n} -  p_{i, t + 1} \Bigm| p_{i, t} \right]} &\leq \left( 1 - \frac{1}{c n} - p_{i, t} \right) + \frac{1}{K c n} \\
            &= \frac{1}{K c n}.
        \end{align*}
        \item $p_{i, t} \leq 1 - \frac{1}{c n} - \frac{1}{K}$: In this case, Proposition \ref{prop:probability_of_frequency_change} becomes applicable and we obtain:
        \[
            \expe{\left[1 - \frac{1}{c n} - p_{i, t + 1} \Bigm| p_{i, t}\right]} \leq \left( 1 - \frac{1}{c n} - p_{i, t} \right) - \xi \cdot \frac{p_{i, t} (1 - p_{i, t})}{K V_t},
        \]
        where $\xi > 0$ is the constant from Proposition \ref{prop:probability_of_frequency_change} hidden by the $\Theta$-notation and does not depend on $c'$.

        The sampling variance $V_t$ can be bounded from above by:
        \[
            V_t = \sum_{j = 1}^n{p_{j, t} (1 - p_{j, t})} \leq \sum_{j = 1}^n{(1 - p_{j, t})} = \varphi_t + \frac{1}{c}.
        \]
        Given that we assumed to be in the situation where no failure occurred, we may assume that $p_{i, t} \geq \eta$ and $1 - p_{i, t} \geq 1 - \frac{1}{c n} - p_{i, t}$. This yields the following bound on the drift:
        \begin{align*}
            \expe{\left[1 - \frac{1}{c n} - p_{i, t + 1} \Bigm| p_{i, t}\right]} &\leq \left( 1 - \frac{1}{c n} - p_{i, t} \right) - \xi \cdot \frac{\eta (1 - p_{i, t})}{ K (\varphi_t + 1 / c)} \\
            &\leq \left( 1 - \frac{1}{c n} - p_{i, t} \right) \cdot \left(1 - \frac{\xi \eta}{K (\varphi_t + 1 / c)}\right).
        \end{align*}
    \end{enumerate}
    This case distinction now allows us to write the expression for the total drift. Let therefore $B_t \subseteq \{1, 2, \dots, n\}$ be the set of the indices of the bits that are at the upper boundary at iteration $t$. The expected change in potential is: 
    \begin{equation}
    \begin{aligned}
    \label{eq:drift_to_bound_by_eps}
        \expe{\left[\varphi_{t + 1} \Bigm| \varphi_t \right]} &= \sum_{i \in B_t}{\expe{\left[1 - \frac{1}{c n} - p_{i, t + 1} \Bigm| p_{i, t} \right]}} + \sum_{i \not\in B_t}{\expe{\left[1 - \frac{1}{c n} - p_{i, t + 1} \Bigm| p_{i, t} \right]}} \\
        &\leq \sum_{i \in B_t}{\frac{1}{K c n}} + \sum_{i \not\in B_t}{\left( 1 - \frac{1}{c n} - p_{i, t} \right) \cdot \left(1 - \frac{\xi \eta}{K (\varphi_t + 1 / c)}\right)} \\
        &= \frac{\lvert B_t \rvert}{K c n} + \left(1 - \frac{\xi \eta}{K (\varphi_t + 1 / c)}\right) \cdot \sum_{i \not\in B_t}{\left(1 - \frac{1}{c n} - p_{i, t} \right)} \\
        &\leq \frac{1}{K c} + \left(1 - \frac{\xi \eta}{K (\varphi_t + 1 / c)}\right) \cdot \sum_{i \not\in B_t}{\left(1 - \frac{1}{c n} - p_{i, t} \right)} \\
        &\leq \frac{1}{K c} + \left(1 - \frac{\xi \eta}{K (\varphi_t + 1 / c)}\right) \cdot \sum_{i = 1}^n{\left(1 - \frac{1}{c n} - p_{i, t} \right)} \\
        &= \frac{1}{K c} + \left(1 - \frac{\xi \eta}{K (\varphi_t + 1 / c)}\right) \cdot \varphi_t \\
        &= \varphi_t + \frac{1}{K} \cdot \left( \frac{1}{c} - \frac{\xi \eta \cdot \varphi_t}{\varphi_t + 1 / c} \right),
    \end{aligned}
    \end{equation}
    where the second inequality holds because $\lvert B_t \rvert \leq n$, and the third one holds due to the fact that the frequencies are capped at $1 - \frac{1}{c n}$, and, thus, we have $1 - \frac{1}{c n} - p_{i, t} = 0$ for $i \in B_t$.
    
    We will now prove that for any given $\varepsilon > 0$, we can fix the constant $c > 0$ in such a way that $\frac{\xi \eta \varphi_t}{\varphi_t + 1 / c} \geq (1 + \varepsilon) \cdot \frac{1}{c}$ holds for all large enough $\varphi_t$. Intuitively, this should hold, because once $\varepsilon > 0$ is fixed, $\lim_{\varphi_t \to \infty}{\frac{\varphi_t}{\varphi_t + 1 / c}} = 1$, both $\xi$ and $\eta$ are constants that are already fixed, and then $c'$ can be adapted appropriately.

    In particular, we may choose $c = \max{\left\{1, \frac{2 (1 + \varepsilon)}{\xi \eta}\right\}}$. Then, for all $\varphi_t \geq V_t \geq 1$ we obtain
    \begin{align*}
        (1 + \varepsilon) \cdot \frac{1}{c} &= \min{\left\{\frac{(1 + \varepsilon) \xi \eta}{2 (1 + \varepsilon)}, 1 + \varepsilon\right\}} \\
        &\leq \xi \eta \cdot \frac{1}{2} \\
        &\leq \xi \eta \cdot \frac{\varphi_t}{\varphi_t + 1} \\
        &\leq \xi \eta \cdot \frac{\varphi_t}{\varphi_t + 1 / c},
    \end{align*}
    where in the second inequality we used that $f(x)=\frac{x}{x+1}$ is increasing on $[1, \infty)$ and lower-bounded by $f(1)=\frac{1}{2}$ and in the last inequality our choice of $c$.
   
    Let us now fix an arbitrary $\varepsilon > 0$, and take $c = \max{\left\{1, \frac{2 (1 + \varepsilon)}{\xi \eta}\right\}}$. We know from our previous derivation that for $\varphi_t \geq 1$ it holds that $\frac{\xi \eta \varphi_t}{\varphi_t + 1 / c} \geq (1 + \varepsilon) \cdot \frac{1}{c}$. Thus, we can continue bounding Equation \eqref{eq:drift_to_bound_by_eps} for $\varphi_t \geq 1$ by:
    \begin{align*}
        \expe{\left[\varphi_{t + 1} \mid \varphi_t \right]} &\leq \varphi_t + \frac{1}{K} \cdot \left( \frac{1}{c} - \frac{\xi \eta \cdot \varphi_t}{\varphi_t + 1 / c} \right) \\
        &\leq \varphi_t + \frac{1}{K} \cdot \left( \frac{1}{c} - (1 + \varepsilon) \cdot \frac{1}{c} \right) \\
        &= \varphi_t - \frac{1}{K} \cdot \frac{\varepsilon}{c}.
    \end{align*}
    In other words, we have obtained that $\expe{\left[\varphi_{t} - \varphi_{t + 1} \mid \varphi_t \right]} \geq \frac{\varepsilon}{c} \cdot \frac{1}{K}$, if $\varphi_t \geq 1$.

    Recall that for the previous bounds on the drift we have assumed $V_t \geq 1$ (which also implied $\varphi_t \geq 1$). If, however, we are in the case that $V_t < 1$, we can leverage the fact that we are in a series of iterations without failure. In particular, this means that all the frequencies are at least equal to the constant $\eta$, which also implies that there is a constant upper bound $C \geq 1$ on the potential. 
    To see the latter, recall that $V_t=\sum_{j=1}^n p_{j,t}(1-p_{j,t)}  $.  Using $p_{i,t} \ge \eta$ and $V_t<1$, we get \[1>V_t =\sum_{j=1}^n p_{j,t}(1-p_{j,t)} \ge \sum_{j=1}^n \eta (1-p_{j,t)}= \eta (\varphi_t+\frac{1}{c}).\] Therefore any constant $C>\frac{1}{\eta}- \frac{1}{c}$ works.

    We now apply the variable drift theorem, Theorem \ref{thm:variable_drift}, to bound the expected time for the potential $\varphi$ to decrease from any initial value $\varphi \leq n$ to a value $\varphi \leq C$. For this, we will use the (constant) drift function $h(\varphi) = \frac{\varepsilon}{K c}$, since we have established in the paragraphs above that it represents a lower bound on the expected change if the potential is at least $C \geq 1$. 

    While the variable drift theorem only considers the hitting time of 0, we note that it is still acceptable to use this theorem in our case: The process we consider instead is one where all the states with potential $\varphi \in [0, C)$ are merged into a single ``state 0''. In this modified process, the smallest state larger than 0 is $x_\mathrm{min} = C$. This modification can only increase the drift (all iterations that previously reduced a potential above $C$ to one below $C$ now reduce the potential directly to 0 instead), so the drift of this process is still bounded below by $h(\varphi)$ for all states with $\varphi \geq C$.

    Thus, by the variable drift theorem, the expected time until we hit a state with potential 0 in our modified process is:
    \[
    \frac{C}{h(C)} + \int_C^n{\frac{1}{h(x)} \, dx} = \frac{K c}{\varepsilon} + \int_C^n{\frac{K c}{\varepsilon}\,dx} = O(K n).
    \]

    In order to end our proof, consider an iteration where $\varphi \leq C$. Given that we assumed we are not in a failure scenario, all the marginal frequencies of the bits attain values that are at least $\eta$, (recall that $\eta < \frac{1}{2}$ was a fixed constant). The probability of sampling a value of $1$ on all bits simultaneously in this case is minimal in the extreme setting where a maximal number of bits have marginal probabilities equal to $\eta$, and all other bits (except at most one) have marginal probabilities at their upper boundaries. Then, the probability of sampling the optimum in one iteration is at least $\left(1 - \frac{1}{c n} \right)^n \cdot \eta^{\left\lceil \frac{C}{1 - \eta} \right\rceil} = \Omega(1)$. Therefore, in a successful phase the the optimum is found with at least constant probability.
    \qed
\end{proof}

\section{Upper bound on the runtime for the aggressive regime}
\label{sec:logn}

In this section, we analyze the \cGA when $K=\Theta(\log^2 n)$ and prove Theorem \ref{thm:boundslogn}. Throughout the whole section we will consider a capping probability of $\Bar{p}\coloneqq\frac{1}{n\log^c n}$ for some constant $c>0$, but our simulations indicate that the result should also hold for the classical $\frac{1}{n}$ capping probability. We prove our main theorem in this regime for $c=7$, and the proof would go through for any $c \ge 7$. It would be possible to reduce this constant, but we aimed for simpler proofs and have not tried to optimize it. 

As in the previous section, we assume that $K\cdot(\frac{1}{2}-\frac{1}{n\log^c n})$ is an integer, so that the marginal frequencies are always in the set 
\[
\left\{ \frac{1}{n \log^c{n}}, \frac{1}{n \log^c{n}} + \frac{1}{K}, \dots, \frac{1}{2}, \dots, 1 - \frac{1}{n \log^c{n}} - \frac{1}{K}, 1 - \frac{1}{n \log^c{n}} \right\}.
\]
The proof of Theorem \ref{thm:boundslogn} proceeds in four main steps.

First, we will show that due to the high genetic drift, the frequencies essentially start by executing random walks until they reach one of the boundaries. As a consequence, the sampling variance $V_t$ drops from $\Theta(n)$ to $O(\log{n})$ during the first $O(\mathrm{polylog}(n))$ iterations, and then stays below $O(\log{n})$ for the remainder of the optimization time with high probability. We call this initial phase the \emph{burn-in phase}.

Then, given this bound on the sampling variance $V_t = O(\log n)$, we prove that it is unlikely to have frequencies at the upper boundary drop below a constant. This basically ensures that, while frequencies from the lower boundary may reach the upper boundary, the converse does not happen. 

Using standard arguments on random walks and geometric distributions, we then argue that indeed all the frequencies starting from the lower boundary after the burn-in phase reach the upper boundary (at least once) within $O(n \cdot \mathrm{polylog}(n))$ iterations.

This puts us in a situation where $n-O(\mathrm{polylog}(n))$ frequencies are at the upper boundary, and the remaining $O(\mathrm{polylog}(n))$ frequencies are lower bounded by a constant. We refine the analysis for this situation to show that with high probability all $O(\mathrm{polylog}(n))$ frequencies reach the upper boundary while no frequency detaches from the upper boundary, and this process only takes $O(\mathrm{polylog}(n))$ iterations. Hence we finally reach a state where $p_{i,t} = 1-\frac{1}{n\log^c n}$ for all positions $i$, from which the optimum is then sampled with high probability in a single iteration, and that terminates the algorithm.

The first proposition covers the initial burn-in phase, and shows that afterwards the sampling variance stays low for at least a quadratic number of iterations (which suffices for our purposes since we will prove that whp the algorithm terminates in quasi-linear time). We first require a lemma that bounds the time until a given frequency reaches one of the boundaries.

\begin{lemma}[adapted from Lemma 6 in \cite{lengler2021complex}]
    \label{lemma:random_walk_bounds}
    Let $c>0$ be a constant and $K=\omega(1)$, and consider the frequency $p_{i, t}$ of a bit $i$ of the algorithm  $\cga(K, \frac{1}{n \log^c n})$ on \DBV. Let $T$ denote the first time that $p_{i, t}$ reaches one of the boundaries. Then for every initial value $p_{i, 0}$ and all $r \geq 8$, $\expe{[T \mid p_{i, 0}]} \leq 4 K^2 \ln{K}$ and $\prob{[T \geq r K^2 \ln{K} \mid p_{i, 0}}] \le 2^{-\lfloor r/8 \rfloor}$.
\end{lemma}

\begin{proof}
This was proven in~\cite[Lemma~6]{lengler2021complex}. Note that even though the lemma is stated for the case of the \onemax function, the proof does not employ any properties of the function to optimize, but rather gives a general statement for any random walk that has a certain probability of being stationary in each state. This result thus also applies to the \dynBV function, since by Proposition \ref{prop:probability_of_frequency_change} the first display equation in the proof of Lemma 6 in \cite{lengler2021complex} is satisfied in our case as well.
\qed
\end{proof}

\begin{proposition}
\label{prop:log_variance_after_burnin}
For $K=\Theta\left(\log^2{n}\right)$ consider the algorithm $\cga(K, \frac{1}{n\log^7{n}})$ on \DBV. After the first $O\left(K^3 \log{n}\right)$ iterations, with high probability the sampling variance $V_t$ will stay below $O(\log{n})$ for at least $n^2$ consecutive iterations.
\end{proposition}

\begin{proof}
Let us divide the optimization time into phases of length $\kappa \cdot K^3 \log{n}$, where $\kappa > 0$ is a constant that will be fixed later. The idea of the proof is to show that with high probability, during a single such phase, all the frequencies that were not at the boundaries at the start of the phase will return to one of the boundaries, and the number of frequencies that detach from the boundaries during that same phase is in $O(\log{n})$. Then, this means that with the exception of the first phase, all other phases will have at most $O(\log{n})$ frequencies that are not at any boundary, which implies that the variance will stay within $O(\log{n})$ for the rest of the optimization time. This follows from the fact that the variance of any frequency at the boundary is $\frac{1}{n \log^7{n}} \left( 1 - \frac{1}{n \log^7{n}} \right) < \frac{1}{n \log^7{n}}$, and the variance of a frequency not at a boundary is at most $\frac{1}{2} \cdot \left( 1 - \frac{1}{2} \right) = \frac{1}{4}$. Thus, an upper bound on the variance is $n \cdot \frac{1}{n \log^7{n}} + O(\log{n}) \cdot \frac{1}{4} = O(\log{n})$.

Consider an arbitrary phase. Our goal is to show that all the bits that are not at a boundary at the start of the phase will eventually reach one of the boundaries by the end of the $ \kappa \cdot K^3 \log{n}$ iterations with high probability. Then, we can conclude the proof by showing that the number of frequencies that detach from a boundary during this same phase is within $C \cdot \log{n}$, where $C > 0$ is a constant that will be fixed later.

To prove the first part, let us fix an arbitrary $i \in \{1, 2, \dots, n\}$ and assume that the current phase we are interested in has started at some iteration $t > 0$. Then, it follows from Lemma \ref{lemma:random_walk_bounds} that with probability at least $\frac{1}{2}$, there will be an iteration $t' \in [t, \, t + 8 \cdot K^3]$ for which $p_{i, t'} \in \left\{\frac{1}{n \log^7{n}}, \, 1 - \frac{1}{n \log^7{n}} \right\}$. By segmenting our phase of $\kappa \cdot K^3 \log{n}$ iterations into intervals of length $8 \cdot K^3$, we obtain that the probability frequency $i$ never reaching one of the boundaries within the phase is at most $\left(1 - \frac{1}{2} \right)^{\left\lfloor \kappa \cdot K^3 \log{n} / (8 \cdot K^3) \right\rfloor} = \left( \frac{1}{2} \right)^{\left\lfloor \kappa \log{n} / 8 \right\rfloor}$. By choosing $\kappa := 40$ this number is at most $n^{-3.4}$ for sufficiently large $n$, 

and via a union bound over the first $n^2$ phases and $n$ frequencies, with high probability all bits that start any phase off-boundary will return to one of the boundaries within $\kappa \cdot K^3 \log{n}$ iterations .

What is now left to show is that there exists a constant $C > 0$ such that at most $C \cdot \log{n}$ frequencies detach from the boundaries during any phase of $\kappa \cdot K^3 \log{n}$ iterations with high probability.

Let the random variable $N$ denote the number of bits that detach from any of the two boundaries during the $\kappa K^3 \log{n}$ iterations of the phase. $N$ is then stochastically dominated by a sum of $n \cdot \kappa K^3 \log{n}$ independent Bernoulli-distributed random variables with parameter $\frac{2}{n \log^7{n}}$. This is because in each of the iterations, at most $n$ bits can detach, and the probability of detaching from a boundary is bounded by $\frac{2}{n \log^7{n}}$: A frequency can only detach if the two bits in the offspring differ, and for a frequency at a boundary, this happens with probability at most $2 \cdot \frac{1}{n \log^7{n}} \cdot \left( 1 - \frac{1}{n \log^7{n}} \right) < \frac{2}{n \log^7{n}}$.

Hence, in expectation at most $\expe{[N]} \leq 2 \cdot \frac{n \cdot \kappa K^3 \log{n}}{n \log^7{n}} = 2\kappa \cdot \frac{K^3}{\log^{6}{n}} = O(1)$ frequencies detach during the phase, where we used $K = \Theta\left(\log^2{n} \right)$. Since $N$ is stochastically dominated by a binomial distribution $\mathrm{Bin}\left(n \cdot \kappa K^3 \log{n}, \, \frac{2}{n \log^7{n}}\right)$, we can employ Chernoff bounds to get $\prob{[N \geq C \log{n}]} = e^{-C \cdot \Omega(\log{n})}$. As in the first part of this proof, we can now choose a suitable $C$ that will allow us to perform a union bound over all phases. Hence, with high probability, no phase has more than $C \cdot \log{n} = O(\log{n})$ bits detaching from any boundary during the entire optimization time.
\qed    
\end{proof}

In the next lemma, we show that frequencies that reach the upper boundary stay above a constant for the following $n^\gamma$ iterations with high probability, where $\gamma$ is a constant that we are free to choose (we will later pick $\gamma=2$). The proof is a fairly straightforward application of the negative drift theorem (Theorem \ref{thm:negative_drift}).

\begin{lemma}
\label{lemma:bits_at_ub_stay_there}
    Let $\Bar{p}\in(0,\frac{1}{2})$ be arbitrary and let $K\ge1$, and consider the algorithm $\cga(K, \Bar{p})$ on \DBV.
    Let $\Bar{p} < \alpha < \beta < 1-\Bar{p}$ and $\gamma > 0$ be constants. Assume that $V_t = O(\log{n})$ holds for the variance all throughout the optimization time. Then there exists a constant $c' > 0$ (possibly depending on $\alpha$, $\beta$, and $\gamma$) such that for a specific bit the following holds: If the bit has marginal probability at least $\beta$ and $K \geq c' \cdot \log^2{n}$, then the probability that during the following $n^\gamma$ iterations the marginal probability decreases below $\alpha$ is at most $O(n^{- \gamma})$.
\end{lemma}

\begin{proof}

The proof of this lemma follows very closely that of Proposition \ref{prop:bound_stay_between_constants}. The only difference is in the lower bound on the drift, where we here use our stronger bound $V_t = O(\log n)$, while the proof of Proposition \ref{prop:bound_stay_between_constants} uses the general bound $V_t \le \frac{n}{4}$.

We first fix an arbitrary frequency $i$ and look at the scaled stochastic process $X_t \vcentcolon= K p_{i, t}$.

In order to apply Theorem \ref{thm:negative_drift}, we choose the interval bounds $a \vcentcolon= K \alpha$ and $b \vcentcolon= K \beta$, and therefore also get $l \vcentcolon= b - a = K (\beta - \alpha)$. To establish the first condition of the theorem, we scale the result from Proposition \ref{prop:probability_of_frequency_change} by $K$ and obtain the following bound on the drift (thus, also fixing $\epsilon \vcentcolon= \xi \alpha (1 - \beta)$ in the meantime):
\begin{align*}
    \expe{[X_{t + 1} - X_t \mid \mathcal{F}_t \, ; \, a < X_t < b]} &\geq K \cdot \xi \cdot \frac{p_{i, t} (1 - p_{i, t})}{K \cdot \max{\{1, V_t}\}} \\
    &\geq \frac{\xi \alpha (1 - \beta)}{C \log{n}},
\end{align*}
where $\xi > 0$ is the constant from the $\Theta$-notation in Proposition \ref{prop:probability_of_frequency_change}, and the last inequality uses the fact that by assumption $V_t \leq C \cdot \log{n}$ for some constant $C > 0$.

The second condition of the theorem can be trivially fulfilled by choosing $r \vcentcolon= 2$. We note that for $j = 0$ the condition obviously holds, and for $j \geq 1$ we have that $\prob{[\lvert X_{t + 1} - X_t \rvert \geq 2j]} = 0 \leq e^{-j}$, because $\lvert X_{t + 1} - X_t \rvert \leq 1$ in the scaled process, as the marginal frequencies change by at most $\frac{1}{K}$.

We now move on to verifying the last condition of the theorem using the fact that $K \geq c' \cdot \log^2{n}$:
\begin{align*}
    \frac{\epsilon l}{132 \log{(r / \epsilon)}} &= \frac{\xi \alpha (1 - \beta)}{C \log{n}} \cdot K (\beta - \alpha) \cdot \frac{1}{132 \cdot  \log{\left( (2 C \log{n}) / (\xi \alpha (1 - \beta)) \right)} }\\
    & \geq \frac{\xi \alpha (1 - \beta) (\beta - \alpha)}{132 C \cdot  \log{\left( (2 C \log{n}) / (\xi \alpha (1 - \beta)) \right)} } \cdot c' \cdot \log{n}.
\end{align*}
Note that $\xi > 0$ can be chosen small enough such that the term above is positive for any choice of $\alpha$ and $\beta$. As all $\xi, \alpha, \beta > 0$ are constants, the term above is greater or equal to $4 = r^2$ for large enough $n$, since $\frac{\log{n}}{\log{\log{n}}} = \omega(1)$.

Furthermore, note that the constant $c' > 0$ in $K \geq c' \cdot \log^2{n}$ can further be adapted to fulfill the following inequality:
\begin{align*}
    \frac{\epsilon l}{132 r^2} &= \frac{\xi \alpha (1 - \beta) \cdot K (\beta - \alpha)}{132 \cdot 4 \cdot C \log{n}} \\
    &\geq \frac{\xi \alpha (1 - \beta) (\beta - \alpha)}{528 C} \cdot c' \cdot \log{n} \\
    &\geq \gamma \ln{n},
\end{align*}
because all other quantities which do not depend on $n$ are fixed constants.

Since, by assumption, we started with $X_0 \geq b$, we can establish via the negative drift theorem that $\prob{[T \leq n^\gamma]} = O \left( n^{- \gamma} \right)$, where $T$ is first time the fixed marginal frequency goes below $\alpha$. \qed
\end{proof}

The next proposition bounds the expected time until a given frequency, starting from the lower boundary, reaches the upper boundary.

\begin{proposition}
\label{prop:bit_goes_up}
Let $c>0$ be a constant and $K\ge1$, and consider the algorithm $\cga(K, \frac{1}{n\log^c{n}})$ on \DBV.
Let $i \in \{1, \dots, n \}$ be an arbitrary bit at the lower boundary, and assume that $V_t = O(\log{n})$ for the rest of the optimization. Then, the expected number of iterations until the frequency of bit $i$ reaches the upper boundary is in $O(K^4 n \log^{c}{n})$.
\end{proposition}

\begin{proof}
Fix a frequency $i$ and assume that it is at the lower boundary.
The proof proceeds in three steps. In the first step, we lower bound the probability that frequency $i$ leaves the lower boundary in a given iteration. In the second step, we couple the random walk $p_{i,t}$ of this frequency to the fair coin gambler ruin's random walk with self-loops, which allows us to leverage the classical bounds on this problem. Finally, we use Lemma \ref{lemma:random_walk_bounds} to deal with the self-loops and bound the number of iterations needed for the gambler's ruin random walk to terminate.

Let $D_{i, t}$ be the event that frequency $i$ detaches from the lower boundary at iteration $t$, and $U_{i, t}$ be the event that, assuming $p_{i,t}=\frac{1}{K}$, the next boundary it reaches (after an arbitrary number of iterations) is the upper boundary. As a first step of the proof, we will be interested in showing that for any $t$:
\[
\prob{[D_{i, t} \cap U_{i, t + 1}]} = \Omega\left( \frac{1}{K n \log^{c}{n}} \right).
\]
This can be shown by writing $\prob{[D_{i, t} \cap U_{i, t + 1}]} = \prob{[D_{i, t}]} \cdot \prob{[U_{i, t + 1} \mid D_{i, t}]}$ and inspecting the two factors separately.

A simple lower bound for $\prob{[D_{i, t}]}$ can be obtained as follows: Since frequency $i$ is at a lower boundary, its value can only increase. This happens exactly when different values for bit $i$ are sampled in the two offspring, and the winning offspring is the one that sampled bit $i$ valued at 1. Moreover, the probability that the winning offspring is the one which has sampled the 1 is at least $\frac{1}{2}$. This yields the lower bound
\[
\prob{[D_{i, t}]} \ge 2 \cdot \frac{1}{n \log^c{n}} \left(1 - \frac{1}{n \log^c{n}} \right) \cdot  \frac{1}{2} = \Omega\left( \frac{1}{n \log^c{n}} \right).
\] 
For the factor $\prob{[U_{i, t + 1} \mid D_{i, t}]}$ we will inspect the following process that dominates ours stochastically. Let us consider the gambler's ruin problem, where the probability of transitioning to the adjacent states is $\frac{1}{2}$ each. This is a pessimistic scenario, as we have already shown in Proposition \ref{prop:probability_of_frequency_change} that there is a (possibly small) positive drift towards higher values of the marginal frequencies: by that proposition, $\prob{[p_{i,t'+1} = p+1/K \mid p_{i,t'} = p]} \ge \prob{[p_{i,t'+1} = p-1/K \mid p_{i,t'} = p]}$ for all $p < 1 - 1/(n\log^c n)$ and all $t' \ge t$. But in this gambler's ruin process, we reach the upper boundary with probability $\frac{1}{K}$ \cite{coolidge1909gambler}, therefore we can conclude that $\prob{[U_{i, t + 1} \mid D_{i, t}]} = \Omega\left(\frac{1}{K}\right)$. Note that since we are only interested in the probability of reaching the upper boundary (and not in the number of iterations within this happens), ignoring self-loops in our random walk is admissible.

We will now inspect the frequency of bit $i$, that is at the lower boundary, and split the iterations until it reaches the upper boundary into phases. A phase starts at the iteration $t'$ when $p_{i, t'}$ is at the lower boundary, and ends at the first iteration $t'' > t'$ for which $p_{i, t''} \in \left\{\frac{1}{n \log^c{n}}, \, 1 - \frac{1}{n \log^c{n}}  \right\}$. Thus, a phase contains all iterations $t \in (t', t'')$ in which the frequency of bit $i$ is at neither of the two boundaries. We call a phase successful if at its end the frequency of bit $i$ is at the upper boundary, otherwise we deem it unsuccessful. 

Above, we have shown that the probability of a phase being successful is in $\Omega\left( \frac{1}{K n \log^c{n}} \right)$. We therefore expect to go through $O(K n \log^c{n})$ phases until a success is encountered.

All that is left to show for us to reach our desired conclusion, is that the expected duration of a phase is $O(K^3)$ iterations. But this follows immediately from Lemma \ref{lemma:random_walk_bounds}, since
\[
4K^2 \ln{K} = O(K^3).
\]  \qed
\end{proof}

We are now ready to prove Theorem \ref{thm:boundslogn}. The proof first consists of putting Proposition \ref{prop:log_variance_after_burnin}, Lemma \ref{lemma:bits_at_ub_stay_there} and Proposition \ref{prop:bit_goes_up} together, in order to reach a situation where almost all frequencies are at the upper boundary, and the remaining frequencies are lower bounded by a constant. From there, we just use a union bound over a series of high probability events that together imply the termination of the algorithm within an additional $\mathrm{polylog}(n)$ iterations. For convenience, we restate the main theorem before starting the proof.

\begin{theorem}
\label{thm:boundslogn_restated}
For $K = \Theta(\log^2 n)$ consider the algorithm $\cga( K, \frac{1}{n\log^7{n}})$ on \DBV. Then there exists a constant $c'$ such that if $K\ge c'\log^2n$, then the optimum is sampled in $O(n \cdot \mathrm{polylog}(n))$ iterations with high probability.
\end{theorem}

\begin{proof}
From Proposition \ref{prop:log_variance_after_burnin} we know that for this chosen configuration of boundaries, with high probability we will have $V_t = O(\log{n})$ for the next $n^2$ iterations after a burn-in phase that lasts $O(K^3 \log{n}) = O(\log^7{n})$ iterations.

We want to now show that, eventually, all frequencies will reach the upper boundary and not drop below a constant. By applying Lemma \ref{lemma:bits_at_ub_stay_there} with $\gamma=2$ and taking a union bound over all $n$ frequencies, we know that by choosing $c'$ large enough then with high probability no frequency that ever reaches the upper boundary will drop below $\alpha = \frac{3}{4}$ (this value for $\alpha$ is just an arbitrary constant) during the following $n^2$ iterations.

Now, we can use Markov's inequality and Proposition \ref{prop:bit_goes_up} to claim that there is a constant probability for one specific bit $i$, which started at the lower boundary, to reach the upper boundary within $O(K^4 n \log^7{n})$ iterations (note that the proof of Proposition \ref{prop:log_variance_after_burnin} implies that every bit will reach either of the boundaries after the burn-in phase, as all of them start out in the middle). Thus, by using arguments similar to the ones in the proof of Proposition \ref{prop:log_variance_after_burnin}, with high probability bit $i$ will have reached the upper boundary after $O(K^4 n \log^7{n} \cdot \log{n}) = O(n \log^{16}{n})$. Moreover, the $O$-notation again allows us to pick suitable constants such that a union bound over all $n$ frequencies will still yield that all of them will have reached the upper boundary after $O(n \log^{16}{n})$ iterations with high probability.

Once all the frequencies have reached the upper boundary at least at some point, then the variance bound $V_t = O(\log n)$ implies that all but $O(\log^3 n)$ of the frequencies are at the upper boundary. Indeed, any frequency $p_{i,t}$ that is not at the upper boundary must be in the interval $[3/4, 1-1/K]$, and in particular contributes at least $p_{i,t}(1-p_{i,t}) \ge \Omega(1/K)$ to the variance $V_t$. Hence, there can be at most $O(K\log n) = O(\log^3 n)$ such frequencies. Let $C \ge 0$ be such that the number of frequencies that are not at the upper boundary is $C \log^3{n}$, and without loss of generality assume that these frequencies are indexed by $i = 1, 2, \ldots, C \log^3{n}$. Since we are in a situation where none of the frequencies will drop below $3/4$, we know that if the frequencies return to a boundary, they will return to the upper one.

We claim that from this point on, with high probability the optimum is sampled during the next $\log^6 n$ iterations. Consider the following events: for $i \in \setof{1, 2, \dots, C \log^3{n}}$, let $R_i$ denote the event that frequency $i$ reaches the upper boundary during this phase of $\log^6{n}$ iterations. Additionally, consider the event $D$ that no frequency that is currently at the upper boundary or reaches the upper boundary detaches from it during this phase.

By Lemma \ref{lemma:random_walk_bounds}, we know that $\prob{[R_i]} \geq 1 - \frac{1}{n}$, since $\log^6{n} =\omega(\log{n} \cdot K^2 \ln{K})$ and hence we can take $r_{\ref{lemma:random_walk_bounds}} = A \log n$ for an arbitrarily large constant $A$. Hence, by a union bound over the $C \log^3{n}$ bits that are not at the upper boundary, we have
\[
    \prob{\left[\bigcap_{i = 1}^{C\log^3{n}}{R_i}\right]} \geq 1 - O\left(\frac{\log^3{n}}{n}\right).
\]
On the other hand, by a union bound over all $n$ frequencies and all $\log^6 n$ iterations, $\prob{[D]} \ge 1 - n\log^6 n \cdot O\left(\frac{1}{n\log^7 n}\right) = 1 - O\left(\frac{1}{\log{n}}\right)$, since the probability that one given bit detaches from the upper boundary during one given iterations is $O\left(\frac{1}{n\log^7 n}\right)$.

Hence  $\prob{\left[D \cap \bigcap_{i = 1}^{C \log^3{n}}{R_i}\right]} \ge 1-o(1)$, and on the event $D \cap \bigcap_{i = 1}^{C \log^3{n}}{R_i}$ we have at least one iteration where all the frequencies are at the upper boundary. When all frequencies are at the upper boundary, then the probability of sampling the optimum in that iteration is at least (since we are sampling two offspring)
\begin{align*}
    \left(1 - \frac{1}{n \log^7{n}}\right)^n &= \left(\left(1 - \frac{1}{n \log^7{n}}\right)^{-n \log^7{n}}\right)^{-1 / \log^7{n}} = 1 - o(1).
\end{align*}
The RHS tends to 1 as $n \to \infty$, because $\lim_{n \to \infty}{\left(1 - \frac{1}{n \log^7{n}}\right)^{- n \log^7{n}}} = e$ and $\lim_{n \to \infty}{\frac{1}{\log^7{n}}} = 0$.
\qed
\end{proof}

\section{Simulations}
\label{sec:sim}
In this section, we provide simulations that complement our theoretical analysis. All figures depict the optimization of \DBV but for varying hypothetical population size $K$. \footnote{For $6 \le K \le 420$ all integer values of $K$ are simulated, for $421 \le K \le 1000$ all integer multiples of $5$, for $1001 \le K \le 6000$ integer multiples of $20$, for $ 6001 \le K \le 10000$ integer multiples of $500$.} The dimension of the search space is always $n=300$. 
The probabilities $p_i, i=1,\dots, n$ are initialized with $\frac{1}{2}$ as in the pseudocode of the algorithm. The lower and upper boundary are set at $\frac{1}{n}$ and at $1-\frac{1}{n}$ respectively. The algorithm stops when the optimum has been sampled, or after $200'000$ iterations if the optimum has not been sampled at that point. The code for the simulations is provided on request.

\begin{figure}
    \centering

         \subfloat{\includegraphics[width=0.5\textwidth]{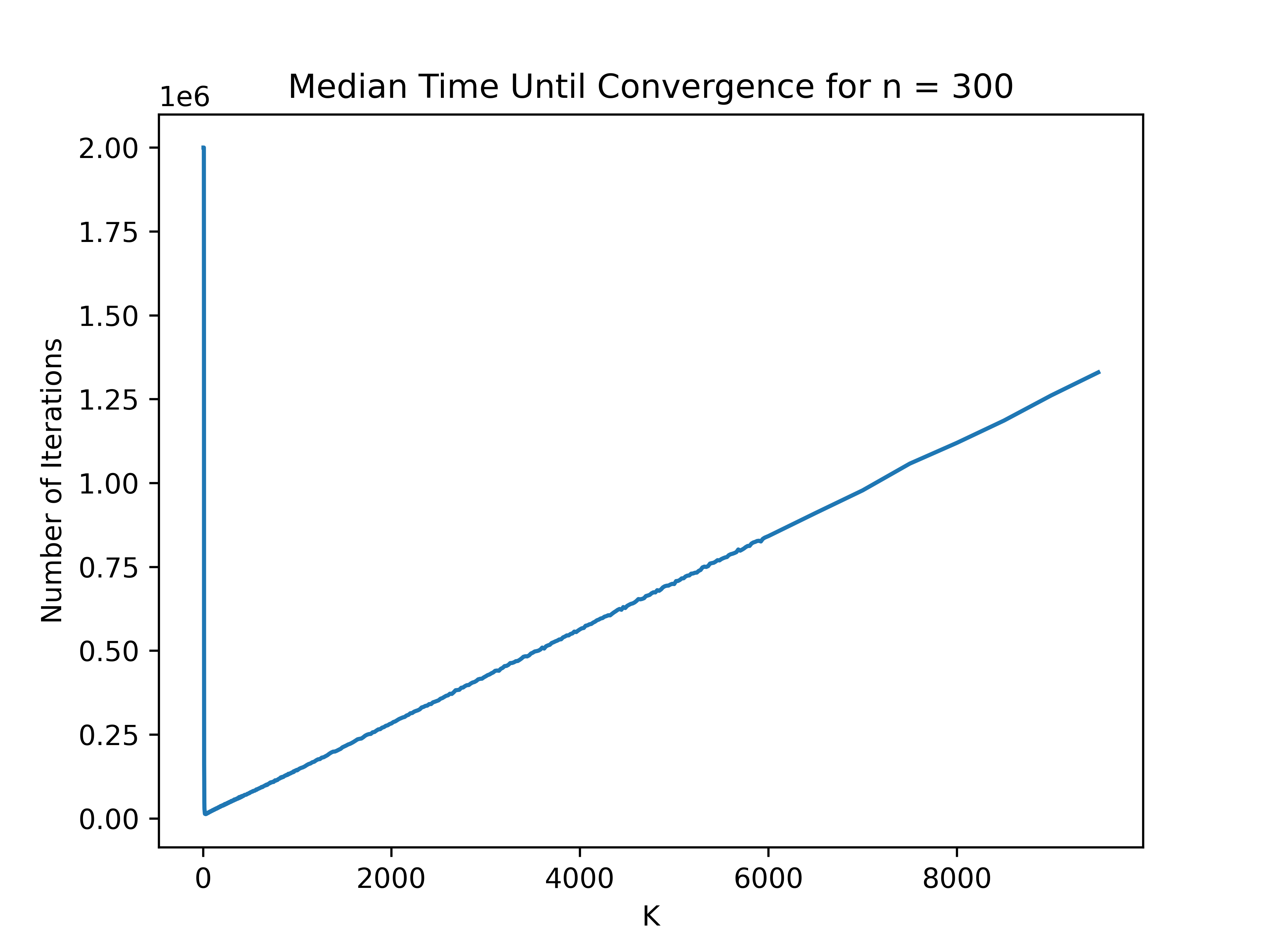}}
        \subfloat{\includegraphics[width=0.5\textwidth]{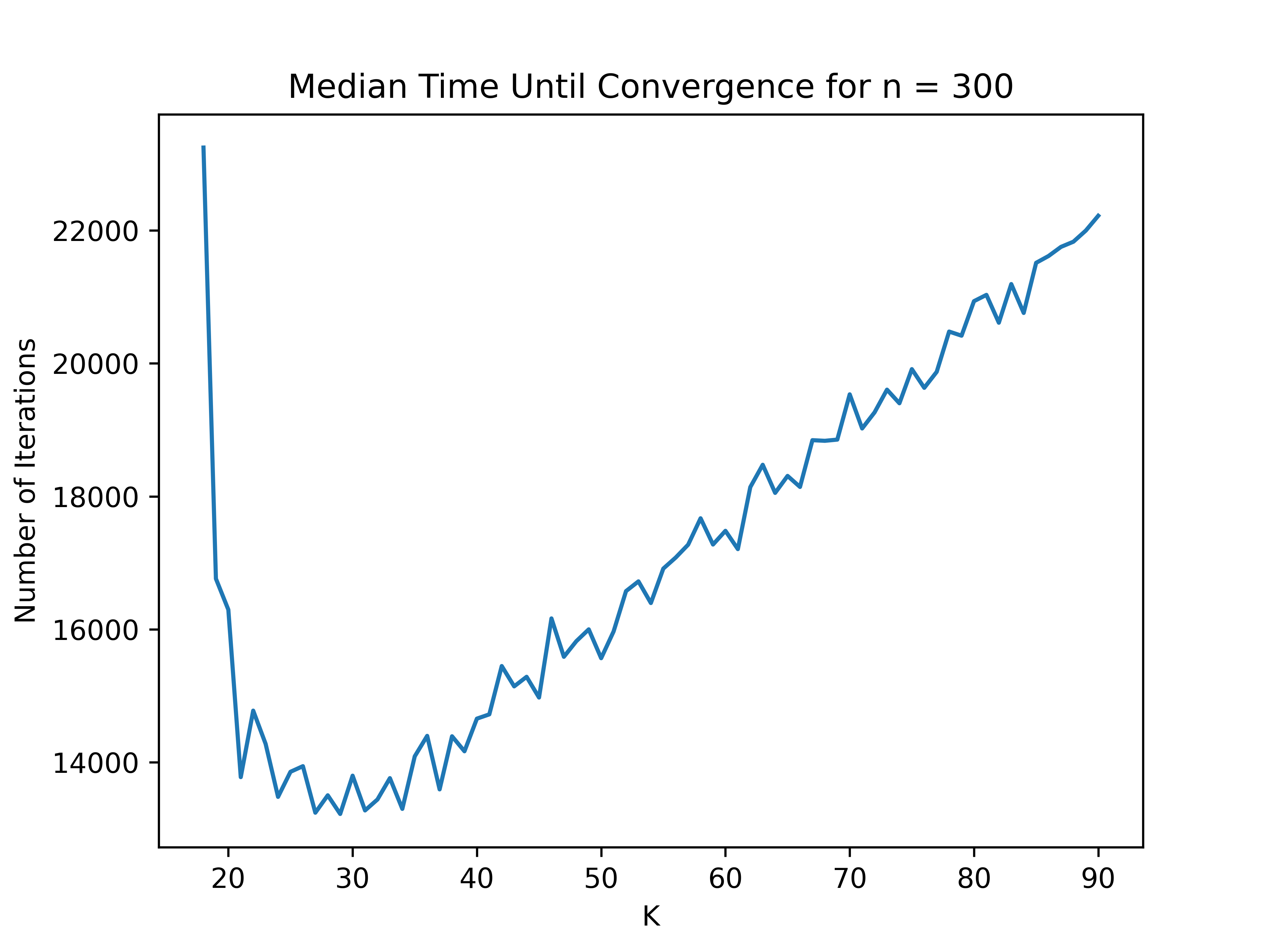}}

    \caption{Number of iterations for the optimization of \dynBV with the \cga when  $6\le K \le 10000$. The right plot shows the subinterval $18 \le K \le 90$. The median over 50 runs is plotted.}
    \label{fig:Kproportional}
   
\end{figure}

The regime of small population sizes is shown in the right plot of Figure~\ref{fig:Kproportional}. Even at a small search space dimension of $n=300$, the asymptotic speed-up of small $K$ is clearly visible. For $K=6,7$, the optimum is not reached before the number of iterations are capped. This is in line with the observation that even for  \onemax, when $K=o(\log (n))$, the runtime of the \cga becomes exponential. However, for $10 \le K \le 20$ we observe a phase transition, with the minimal runtime attained for hypothetical population sizes $K$ around $30$.
Due to the small problem dimension, it is difficult to tell if the threshold is located at $K=\Theta(\log n)$, at $K=\Theta(\mathrm{polylog}(n))$, or even $K=\Theta(n^c)$ for some small $c<1$. But the data is consistent with the theoretical result that the optimum is obtained for the sublinear $K$ regime of genetic drift.

For $K = \Omega(n\log n)$, Theorems~\ref{lemma:slow_progress} and~\ref{thm:intro_upper_bound_runtime} show an asymptotically tight runtime bound of $\Theta(Kn)$. Figure~\ref{fig:Kproportional} covers a range of $K=6$ up to $K=10000$, exceeding the search space dimension of $n=300$ by 2-3 orders of magnitude. We see indeed that the runtime increases proportionally to $K$, thus confirming our theoretical findings. In particular we see that, contrary to the optimization of \onemax, there are no local minima after the transition from the exponential to the polynomial regime. Furthermore, the plot indicates that the runtime scales linearly with $K$ in practice much earlier than our theoretical bound from Theorem~\ref{thm:intro_upper_bound_runtime}.  

\begin{figure}
    \centering
    \includegraphics[width=0.7\textwidth]{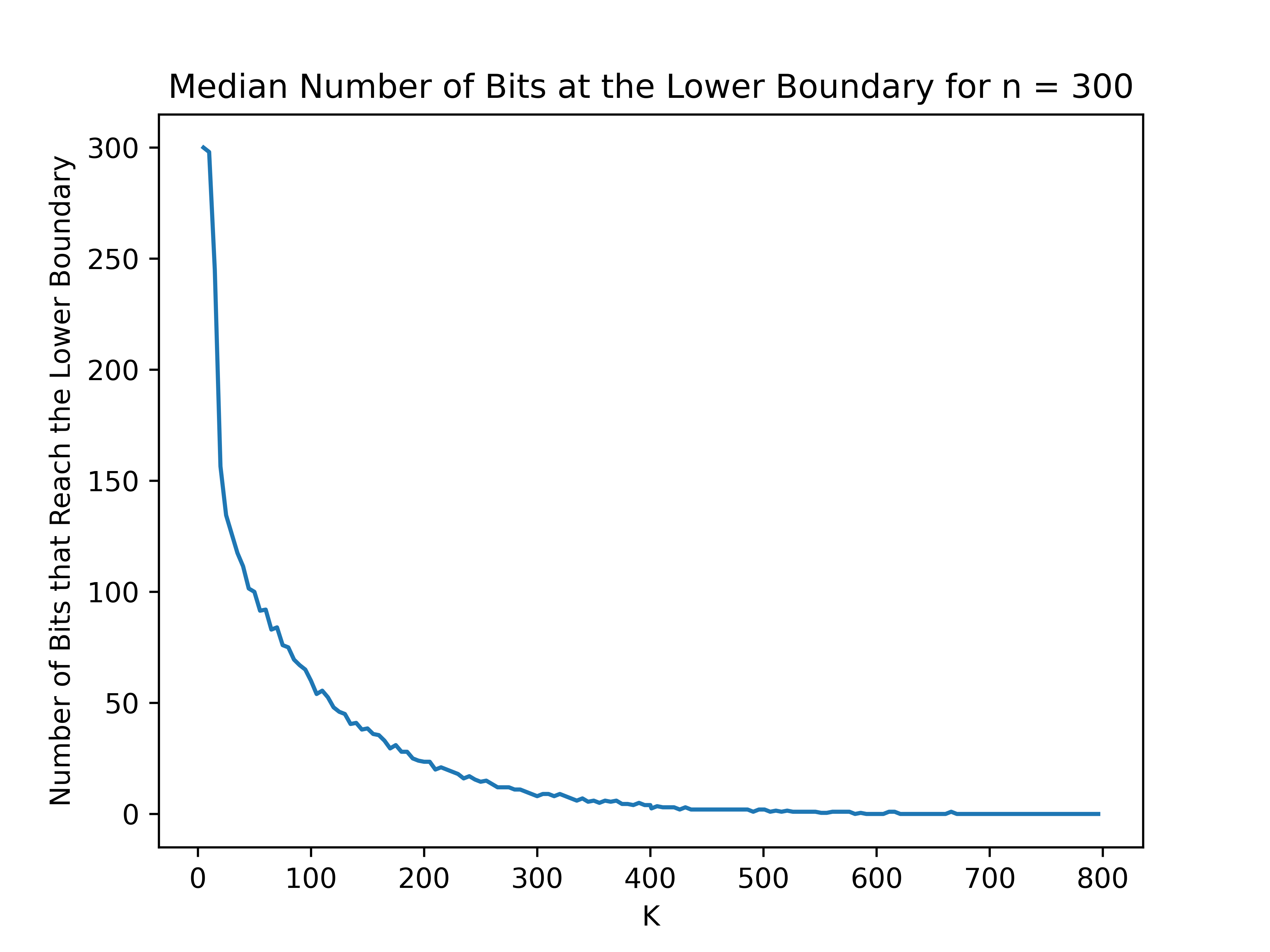}
    \caption{Number of bits that reach the lower boundary $1-\frac{1}{n}$ for the range $5 \le K \le 800$. The median over 20 runs is plotted.}
    \label{fig:boundarybits}
\end{figure}

In Figure~\ref{fig:boundarybits}, we see that after an initial exponential decrease, which is similar to the initial exponential runtime decrease in Figure~\ref{fig:Kproportional}, the number of frequencies ever reaching the lower boundary tapers off only slowly. In particular, for the empirically optimal value $K\approx 30$ from Figure~\ref{fig:Kproportional} still many frequencies reach the lower boundary, confirming that this is in the aggressive regime of strong genetic drift. Until $K=n$, there is still a double-digit number of bits which reach the lower boundary. Only after approximately  $K=500 = \frac{5}{3} n$ the median drops to zero.

\begin{credits}
\subsubsection{\ackname} M.K. and U.S. were supported by the Swiss National Science Foundation [grant number 200021\_192079]. The Dagstuhl seminar 22182 ``Estimation-of-Distribution Algorithms: Theory and Applications'' gave inspiration for this work.

\subsubsection{\discintname}
The authors have no competing interests to declare that are
relevant to the content of this article. 
\end{credits}
%
%

\bibliographystyle{splncs04}
\bibliography{references}

\end{document}